\documentclass[twoside]{article}
\usepackage[preprint]{aistats2027}
\usepackage{amsmath,amssymb,amsthm,bm,graphicx}
\usepackage[round]{natbib}
\usepackage{booktabs}
\usepackage{url}
\usepackage[hidelinks]{hyperref}
\usepackage[capitalize,noabbrev]{cleveref}
\usepackage{xcolor}
\usepackage{graphicx}
\graphicspath{{figures/}}

\definecolor{emcol}{HTML}{1f77b4}
\definecolor{madcol}{HTML}{d62728}
\definecolor{healcol}{HTML}{2ca02c}
\definecolor{rescol}{HTML}{5E3C99}
\definecolor{blindcol}{HTML}{E66101}

\newcommand{\R}{\mathbb{R}}
\newcommand{\KL}{\mathrm{KL}}

\newcommand{\bJ}{{\bm{J}}}
\newcommand{\bM}{{\bm{M}}}

\newcommand{\bG}{{\bm{G}}}
\newcommand{\bS}{{\bm{S}}}
\newcommand{\E}{\mathbb{E}}
\newcommand{\given}{\,|\,}
\newcommand{\bx}{{\bm{x}}}
\newcommand{\by}{{\bm{y}}}
\newcommand{\bA}{{\bm{A}}}
\newcommand{\bF}{{\bm{F}}}
\newcommand{\bN}{{\bm{N}}}
\newcommand{\bP}{{\bm{P}}}
\newcommand{\truth}{p_\star}
\newcommand{\reg}{\rho}
\newcommand{\loopop}{\mathcal{T}}
\newcommand{\blind}{\mathcal{B}}

\newtheorem{theorem}{Theorem}
\newtheorem{proposition}{Proposition}

\newtheorem{corollary}[theorem]{Corollary}
\theoremstyle{remark}
\newtheorem{remark}{Remark}
\theoremstyle{plain}

\begin{document}

\runningtitle{Priors learned from legacy reconstructions inherit undetectable overconfidence}

\runningauthor{Ali Siahkoohi, Sina Alemohammad}

\twocolumn[

\aistatstitle{Priors learned from legacy reconstructions inherit
undetectable overconfidence}

\aistatsauthor{ Ali Siahkoohi \And Sina Alemohammad }
\aistatsaddress{ Department of Computer Science \\ University of Central Florida \\ \texttt{alisk@ucf.edu}
\And Department of Electrical and Computer Engineering \\ The University of Texas at Austin \\ \texttt{sina.alemohammad@austin.utexas.edu} } ]

\begin{abstract}
Where truths are scarce (e.g., seismic and medical imaging), learned priors in ill-posed inverse problems are trained on archives of legacy reconstructions---i.e., an older method's outputs---and their reported uncertainty is taken as data-driven. We show that this prior is, in the population limit, exactly the regularizer that produced its archive of posterior samples, advanced one expectation--maximization step toward the truth. While the step improves the regularizer on the directions the measurements resolve, it leaves the regularizer's assumption on the operator's blind subspace unchanged. An archive of single-best reconstructions collapses the blind interval to zero width. Neither error is detectable in practice, as truths differing only on the blind subspace share the data law, and simulation-based calibration is neutral by construction. We identify from the operator alone which directions the measurements do not inform, and, given a handful of ground-truth models, build intervals there that contain the truth as often as they claim to. We validate these findings on a two-dimensional example with closed-form predictions and in controlled experiments on seismic-imaging and groundwater-flow operators, against priors trained on the truth.
\end{abstract}

\section{Introduction}
\label{sec:intro}

Due to their non-uniqueness, ill-posed inverse problems call for a prior. A learned prior (i.e., a deep generative model of plausible solutions) is increasingly used in place of a handcrafted one~\citep{arridge2019solving}, and the posterior it induces yields sharper reconstructions~\citep{jalal2021robust,chung2023dps} with uncertainty read off posterior samples~\citep{SiahkoohiRizzutiHerrmann_2022,OrozcoSiahkoohiLouboutinEtAl_2025}. Training one, however, requires examples of the truth, which seismic and medical imaging can rarely provide. The common approach is to train the prior on an archive of \emph{legacy reconstructions}---i.e., the outputs of an older, regularized reconstruction method---and to treat it as data-driven, its reported uncertainty included~\citep{jalal2021robust,leuschner2021lodopab,SiahkoohiRizzutiOrozcoEtAl_2023}. We call that process---reconstructing each observation with the old method, collecting the results into an archive, and fitting a prior to them---\emph{curation}, and the prior it produces the \emph{curated} prior. Training and evaluating on already-processed data is a documented risk for accelerated MRI~\citep{shimron2022implicit}; what curation does to the reported uncertainty is this paper's subject (\cref{fig:teaser}).

\begin{figure*}[!t]
\centering
\includegraphics[width=1.00\textwidth]{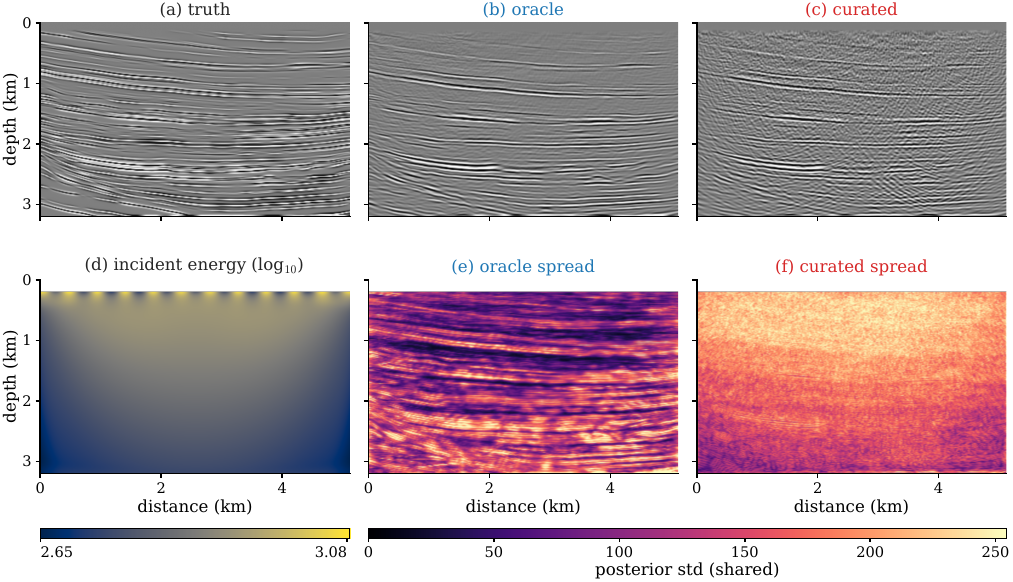}
\caption{Seismic Born imaging on one observation: the truth-trained \textcolor{emcol}{oracle} prior is the control, the legacy-trained \textcolor{madcol}{curated} prior the object of study. (a)--(c)~truth and both posterior means; (d)~incident wavefield energy (log scale), an illumination proxy; (e,f)~pointwise posterior spread, one shared scale, with opposite gradients.}
\label{fig:teaser}
\end{figure*}

The operator of an ill-posed problem splits the model space in two: the directions its measurements resolve, and a \emph{blind subspace} they never reach. That a single posterior relies on its prior along the blind directions is classical~\citep{poirier1998,moon2012bayesian,gustafson2015}. We examine, one step earlier, what fills the prior itself when it is fit to an archive the same operator produced.

The answer has an exact form. In the population limit, an archive of posterior samples is the regularizer that produced it, advanced one expectation--maximization (EM) step toward the truth (Theorem~1). On the blind subspace that step leaves the old assumption frozen where it was (Proposition~1). The harder case is the one practice keeps, since real archives store a single-best (maximum-a-posteriori) reconstruction per observation, e.g., one migrated section per survey or one reconstructed image per patient~\citep{yilmaz2001seismic,zbontar2018fastmri,leuschner2021lodopab}. Such an archive reports zero blind width, certainty where the truth varies, for any penalty (Theorem~2). Nothing available in practice reveals this. Two truths differing only on the blind subspace share the data law, simulation-based calibration is neutral by construction~\citep{talts2018validating}, and no procedure using observations and archive alone reports a finite blind interval whose coverage is guaranteed (\cref{c:honest}). Any correction must therefore rely on information they do not contain.

What we propose to report follows from the same results. The blind directions are identified by the operator alone, before any prior is trained and without a single truth; that report is the \emph{resolvability statement} we provide (Section~\ref{sec:reportcard}). Given a handful of \emph{reference truths}---information the observations do not carry---a split-conformal band contains the truth as often as it claims to, whatever the prior reports, since the prior sets only the band's width and never its validity (Proposition~2). Where a second instrument is available, rebuilding the archive against it de-freezes exactly the directions the pair resolves and no others (\cref{c:augment}), a design statement about what to measure.

We test all of this on a two-dimensional example with closed-form predictions and on two operators from practice, seismic linearized-Born imaging and groundwater flow, each with a single-best legacy archive and a truth-trained control (Figure~\ref{fig:teaser}). Whether the idealized overconfidence survives estimation depends on capacity, since a prior too weak to fit its archive would mask the inherited assumption with spread of its own~\citep{koehler2021representational}. A diffusion model~\citep{ho2020denoising} and a deep coupling flow~\citep{kruse2021hint} are expressive enough, and both recover the pattern the theory predicts. The curated prior's intervals contain the truth less often on the blind subspace than on the resolved directions, the control shows no such gap, and on the seismic operator a random subspace of the same dimension separates neither prior. The gaps widen at an exact kernel and narrow on a low-illumination tail. On the groundwater operator we also perform the added-measurement repair. Related work is next (Section~\ref{sec:related}); we set up the problem in Section~\ref{sec:setup}, identify what fills the gap and why no measurement checks it in Section~\ref{sec:theory}, state what to report in Section~\ref{sec:reportcard}, and test both on the two operators in Section~\ref{sec:experiments}.

\section{Related work}
\label{sec:related}

Priors trained on surrogate data are the practice this paper examines. In seismic imaging, generative priors are trained on archives of migrated or inverted sections~\citep{SiahkoohiHerrmann_2021,SiahkoohiRizzutiLouboutinEtAl_2021,BaldassariSiahkoohiGarnierEtAl_2023,SiahkoohiRizzutiOrozcoEtAl_2023,shi2024generative} and on legacy or benchmark velocity models~\citep{wang2023prior,yin2024wise}; in medical imaging, priors for accelerated MRI and the standard low-dose CT benchmarks are built on reconstructed scans~\citep{jalal2021robust,leuschner2021lodopab}. What these works leave open is what the trained prior reports where the operator is blind, which is the question the archive law (Section~\ref{sec:theory}) answers.

A learned prior can instead be trained against the measurements themselves. The sensing theorems of \citet{tachella2023sensing} characterize when a signal distribution is learnable from incomplete observations. A line of methods applies this idea to generative models~\citep{bora2018ambientgan}: ambient diffusion~\citep{daras2023ambient}, expectation--maximization training from corrupted observations~\citep{bai2024em,rozet2024learning}, and the population corrections of \citet{barco2025correcting}. These methods operate where the corruption leaves no blind subspace, or fill the missing directions with fresh randomness per observation. On one fixed operator, our loop-invariance result is the complementary statement, since no number of rounds against fresh measurements from that operator moves the blind conditional.

The failure mode we study is one no within-distribution check detects, and its neighbors mark the boundary. Simulation-based calibration is neutral by construction~\citep{talts2018validating,hermans2022trust}, and ground-truth-free metrics for learned uncertainty are meaningful exactly where the target is identified from the observables~\citep{pignet2025metrics}, which a blind functional is not. Overconfidence from a misspecified model concentrates the posterior on a wrong value the data could in principle reveal~\citep{shen2025pcuq}; ours is invisible to the data by construction. Training on a model's own outputs collapses variance over many recursive rounds~\citep{alemohammad2024self,shumailov2024curse,suresh2025collapse}; curation is a single population step against real measurements, and what it inherits is the regularizer's assumption rather than sampling noise.

Partial identification is the classical framework~\citep{poirier1998,gustafson2015}, and the standard approach is sensitivity analysis rather than a data-driven report. Our results place curation within that framework, with the prior itself produced by the operator, and the remedy adds information instead. Split-conformal prediction provides distribution-free validity from exchangeable references~\citep{vovk2005}, and conformal methods calibrate learned imaging uncertainty against ground truth~\citep{angelopoulos2022image,teneggi2023trust}, as our band does for blind functionals. Recovery theory for diffusion priors assumes the truth lies on the prior's support~\citep{leong2026recovery}, which the curated prior fails on the blind subspace.

\section{Curation and the blind subspace}
\label{sec:setup}

\paragraph{The operator's split.} An unknown $\bx\in\R^d$ is observed only through indirect, noisy measurements $\by$~\citep{tarantola2005inverse,stuart2010inverse}, where a known forward operator $\bF$ and a known Gaussian noise model define the likelihood
\begin{equation}
p(\by\given\bx)=\mathcal{N}(\bF(\bx),\Gamma),\qquad \Gamma\succ0.
\label{eq:likelihood}
\end{equation}
Only the closed forms need Gaussian noise; the general statements require only that the noise depend on $\bx$ only through $\bF(\bx)$. Every statement below takes a closed form in the linear--Gaussian case $\bF(\bx)=\bA\bx$, where the split is global and exact,
\begin{equation}
\blind=\ker\bA,\qquad \blind^{\perp}=\mathrm{row}(\bA),
\label{eq:split}
\end{equation}
the \emph{blind subspace} and the \emph{resolved subspace}. We write $\bN$ for an orthonormal basis of $\ker\bA$ and split the unknown as $\bx=(\bx_R,\bx_B)$, resolved and blind; the models a measurement cannot distinguish are the operator's blind \emph{fibers}, the affine leaves $\bx+\ker\bA$. For a nonlinear operator the split is global exactly where the operator has a genuine invariance; without one it is read at a base state through the Jacobian, a local statement (appendix). An exact kernel and a noise-level near-null are different objects. An operator whose singular values decay but never vanish has directions that are slowly resolved rather than blind~\citep{knapik2011}, and at realistic record counts a weakly illuminated direction behaves like a kernel direction. The exact statements (Section~\ref{sec:theory}) idealize at the kernel, the near-null forms in the appendix extend them to a decaying spectrum, and there the split is set by a cutoff whose sweep we report.

\paragraph{The legacy pipeline.} The unknown is drawn from a true distribution $\truth$, \emph{the truth}. A classical method combines the likelihood with a handcrafted \emph{regularizer} $\reg$~\citep{engl1996regularization} to form the legacy posterior $\pi_\reg(\bx\given\by)\propto p(\by\given\bx)\,\reg(\bx)$. Three assumptions describe the data-scarce setting: (i)~the truth is never observed directly; (ii)~the operator and the noise model are known, with many real measurements at hand, their law written $\truth(\by)$; and (iii)~for each measurement the legacy pipeline returns either a posterior sample, giving a \emph{posterior-sample archive}, or the single-best reconstruction $\mathrm{MAP}_\reg(\by)$, giving a \emph{single-best archive}. A prior fit to a posterior-sample archive converges, in the population limit of a nonparametric fit, to the data-averaged-posterior map applied to $\reg$,
\begin{equation}
\loopop[\pi](\bx):=\E_{\by\sim\truth(\by)}\big[\pi(\bx\given\by)\big],
\qquad
q_\reg:=\loopop[\reg],
\label{eq:curated}
\end{equation}
the \textcolor{madcol}{\emph{curated prior}} $q_\reg$; a fit to a single-best archive converges instead to the pushforward of the measurement law through $\mathrm{MAP}_\reg$. The \textcolor{emcol}{\emph{oracle prior}}, trained on samples of the truth itself, is the control, available only where the truth can be sampled. On a blind direction the posterior conditional equals the prior's for every $\by$, so the spread reported there is fixed before any measurement is taken. Coverage of credible sets is how the inverse-problems literature assesses such reports~\citep{szabo2015,monard2021bvm}, and it is measured only in controlled settings where the truth is known.

\section{The archive law}
\label{sec:theory}

Averaged over the real measurements, an archive that keeps posterior samples is the handcrafted regularizer advanced one EM step toward the truth and frozen where the data are uninformative. Each statement indicates its scope in its bracketed title, so unmarked hypotheses mean any operator, noise model, and prior for which \eqref{eq:curated} is defined. All proofs are in Appendix~\ref{app:proofs}. The results here and in the appendix are machine-checked in Lean~4 against mathlib, with nothing admitted and nothing domain-specific axiomatized and a few analytic and modeling steps stated as explicit hypotheses (the development records the correspondence statement by statement, and the appendix enumerates the deliberate paper-only remainder).

\subsection{One step from the regularizer}
\label{ss:emstep}

\begin{theorem}[A legacy archive is one EM step from the regularizer \emph{(posterior-sample archive; any operator, noise and full-support regularizer)}]\label{p:law}
Under a posterior-sample archive, for any forward model, noise, and full-support regularizer $\reg$ whose data-marginal divergences are finite (the truth's measurement law from the regularizer's marginal and from the advanced step's), the step's own divergence from $\reg$ finite as well, the legacy reconstructions are independent samples of $q_\reg=\loopop[\reg]$ of~\eqref{eq:curated}, which is exactly the regularizer advanced one EM step for the mixing density $\reg$ against the truth's measurements~\citep{dempster1977em,vardi1985}:
\begin{equation}
q_\reg(\bx)=\reg(\bx)\,\E_{\by\sim\truth(\by)}\!\left[\frac{p(\by\mid\bx)}{\reg(\by)}\right],
\label{eq:emstep}
\end{equation}
with $\reg(\by)=\int p(\by\mid\bx)\,\reg(\bx)\,d\bx$ the regularizer's data marginal---the Vardi/Lucy--Richardson/NPMLE multiplicative update~\citep{richardson1972,lucy1974}. That step never decreases the data-marginal likelihood, which is the precise sense of ``toward the truth,'' not pointwise convergence.
\end{theorem}

The multiplicative update for the NPMLE of a mixing density~\citep{kiefer1956,lindsay1983} is itself classical~\citep{vardi1985}, and so is the observation that its limit is set by the initialization where the solution is not unique~\citep{vardi1993}; the identification is this paper's, since an archive curated by a pipeline that intended no such step is exactly one step of it. The identity is a population idealization, assuming exact posterior sampling and a nonparametric fit; the freeze below requires less of the archive's sampler, and is the more robust half.

\begin{proposition}[On the blind fibers the learned prior is the assumption \emph{(exact; any factoring noise)}]\label{p:blind}
On a blind fiber the curated prior's conditional equals the regularizer's for any forward operator with a blind subspace, under any prior and any noise kernel that depends on $\bx$ only through $\bF(\bx)$; the step moves the reported assumption only across the fibers the instrument separates, never within one, however much real data fed the curation. In the linear--Gaussian case this is the blind-block precision identity
\begin{equation}
\bN^\top\Sigma_{q_\reg}^{-1}\bN=\bN^\top\Sigma_\reg^{-1}\bN,\qquad \mathrm{range}\,\bN=\ker\bA.
\label{eq:freeze}
\end{equation}
\end{proposition}

Curation is thus a projection of learning onto the operator's range---the resolved marginal moves and the blind conditional does not---and rebuilding the archive from fresh measurements of the same operator leaves that conditional unchanged at every round (\cref{c:invariant}).

\subsection{The single-best collapse}
\label{ss:singlebest}

Most archives store a single-best reconstruction~\citep{yilmaz2001seismic,leuschner2021lodopab}, and on the blind fiber the discard is total.

\begin{theorem}[Single-best curation collapses the blind-fiber uncertainty \emph{(any forward operator with a blind subspace, any factoring noise, any variational estimator)}]\label{p:map}
Under a single-best archive let it store any variational reconstruction
\begin{equation}
\hat\bx(\by)\in\arg\min_{\bx}\ D\big(\by,\bF(\bx)\big)+R(\bx),
\label{eq:varest}
\end{equation}
a data-fidelity term plus a penalty, the Bayesian $\mathrm{MAP}_\reg$ among its instances; assume a minimizer exists. Then $\hat\bx(\by)$ is a deterministic function of $\by$, so the curated ensemble contains none of the within-$\by$ posterior spread. On a blind fiber the data term is flat, so the blind component of the reconstruction is the penalty's conditional minimizer given the resolved coordinates---a function of those coordinates alone, never of the measurement; in the split coordinates,
\begin{equation}
\hat\bx_B(\bx_R)\in\arg\min_{\bx_B}R(\bx_R,\bx_B),
\label{eq:mapcollapse}
\end{equation}
a single point with zero within-fiber spread. The conditional credible interval on a blind functional $\theta=\bm v^\top\bx$ ($\bm v$ in the blind subspace) therefore has zero width and never covers a truth whose fiber-conditional law of $\theta$ is nonatomic. \cref{p:map-full} states it in full.
\end{theorem}

The fiber structure of \eqref{eq:mapcollapse} is the classical geometry of minimum-penalty selections~\citep{engl1996regularization}; what the theorem adds is what that structure does to a curated archive. A generative estimator fit to the collapsed archive reports not the data's spread but its own smoothing scale (the smoothing floor,~\cref{r:map-lg}), which is how the collapse is observed in a trained prior.

What the freeze costs also has a closed form. In the linear--Gaussian case, with $s_\reg,s_\star$ the curated prior's and the truth's blind-fiber conditional spreads along a direction and a correctly centered regularizer, the nominal $1-\alpha$ interval covers the truth at
\begin{equation}
C \;=\; 2\,\Phi\!\left(z_{1-\alpha/2}\,\frac{s_\reg}{s_\star}\right) - 1,
\label{eq:coverage}
\end{equation}
overconfident exactly when the inherited assumption is tighter than the truth, and a frozen mean error only exacerbates it (\cref{p:cover}).

\subsection{Undetectability}
\label{ss:undetectable}

No measurement can contradict either reading above. For any noise model depending on the unknown only through the operator, the data marginal depends on the truth only through its resolved marginal, so two truths differing only on the blind subspace induce the identical measurement law and, under single-best curation, identically distributed archives and the identical learned prior (witness family in \cref{c:nonident}). Simulation-based calibration is likewise neutral by construction, since it scores a prior's consistency with its own posterior, never its fidelity to the truth (\cref{r:selfcheck}).

\begin{corollary}[No finite blind interval is an inference \emph{(any procedure; linear--Gaussian witness)}]\label{c:honest}
Fix a blind direction $\bm v$ and let $\mathcal P$ be the truths consistent with the measurement law, which have an arbitrary blind-fiber conditional; the witness family places in $\mathcal P$ truths whose conditional spread $s$ along $\bm v$ is arbitrarily large. Let $\hat I$ be any interval rule for $\theta=\bm v^\top\bx$ computed from the recorded measurements and the archive---learned or classical, randomized or not. Its law is the same for every member of $\mathcal P$, while a witness truth's conditional density along $\bm v$ is at most $(s\sqrt{2\pi})^{-1}$, so
\begin{equation}
\inf_{p\in\mathcal P}\Pr\nolimits_p\big(\theta_\star\in\hat I\big)=0
\qquad\text{whenever }\ \E|\hat I|<\infty,
\label{eq:honest}
\end{equation}
and an interval with coverage $1-\alpha$ uniformly over $\mathcal P$ has infinite expected length~\citep{gleser1987nonexistence}.
\end{corollary}

On the resolved subspace the data still check what the curated prior reports, and outside information can make a blind interval finite, whether a hard physical constraint bounding the witness spread~\citep{backus1970inference,stark1992inference} or reference truths (Section~\ref{sec:reportcard}).

\section{Resolvability and the coverage guarantee}
\label{sec:reportcard}

Identifying the blind directions requires only the operator, while testing what a prior reports there needs a handful of truths. We release both with the code.

\paragraph{The resolvability statement.} Computing the split is classical, in the resolution analysis of \citet{backus1968resolving} and its extensions~\citep{cui2014lis,bhadra2021hallucinations}. The archive law adds the reason to include it with a learned reconstruction, since on its blind side no measurement can check what the prior reports. For an operator small enough to factor, the statement splits the model space by the full singular value decomposition, never a thin-factorization artifact, at a cut set to machine precision for an exact kernel or to the noise level for a sub-noise tail. For an operator given only as code it takes instead a set of directions certified blind by probing, at one operator application per probe.

It returns (i)~the split itself; (ii)~for each functional one intends to report, the fraction of its energy on the blind subspace, hence a classification as resolved, mixed, or blind; and (iii)~for a set of prior or posterior samples, the fraction of the reported variance that no measurement constrains. The output is a set rather than a ranking. The statement identifies which reported intervals are prior-supplied, not whether the prior's spread there is correct. Probing certifies only directions below its illumination cut, so the fraction it reports is a lower bound, the direction of error a diagnostic should have.

\paragraph{The certified band.} A handful of references suffice to test what the prior reports---the occasional fully sampled acquisition, not a training set---their number set by the audit calculation in the appendix and the same however poor the illumination. Split-conformal prediction~\citep{vovk2005} scores $k$ held-out references and the candidate, and under exchangeability the $k+1$ scores are equally likely to arrive in any order. The band of candidates scoring below the appropriate order statistic therefore covers at the reported frequency, regardless of the distribution and of how the score was built.

\begin{proposition}[A certified blind interval, whatever the prior reports \emph{(exchangeable references)}]\label{p:certify}
Let $\bx_1,\dots,\bx_k$ be references exchangeable with the truth reported on and held out of the archive the curated prior was trained on, with scores almost surely tie-free (any density for the references ensures this), and score a candidate value of a blind functional $\theta=\bm v^\top\bx$ by the prior's own standardized blind residual $\eta(\theta)=|\theta-m_\reg(\bx_R)|/s_\reg$. The band $\{\theta:\eta(\theta)\le\eta_{(\lceil(k+1)(1-\alpha)\rceil)}\}$---the \emph{certified band}---is an interval whose coverage of $\theta_\star$ is at least $1-\alpha$, marginally over the joint exchangeable sample of references and truth, for every prior and every truth distribution making the sample exchangeable. Where the truth's resolved and blind coordinates are independent, validity holds conditionally on the observation as well. The band is finite once $k\ge\lceil 1/\alpha\rceil-1$, and below that it is infinite rather than wrong.
\end{proposition}

The certified half-width over the reported one is the reportable quantity---how much wider the truth is than the prior claims, measured rather than assumed---at a cost of $k$ scores and a sort. Where the conditional form holds, no score adapted to the recorded data improves on the prior's own residual~(\cref{app:proofs}). At a small reference count the coverage realized for one reference set spreads around the nominal level, tightening as references accumulate. Exchangeability of the references is itself unverifiable, but unlike what the prior reports it is one explicit assumption about a small reference set, which one can argue about, document, and revisit.

\section{Numerical experiments}
\label{sec:experiments}

We validate the theory first on a stylized two-dimensional example, where every prediction has a closed form (\cref{fig:toy-archives,fig:toy-money}). We then study two operators from practice, seismic linearized-Born imaging and groundwater flow, where the population idealization (Section~\ref{sec:theory}) no longer holds. On the operators we report achieved central-$(1-\alpha)$ coverage against nominal, separately on the resolved and blind subspaces, where coverage below nominal indicates overconfidence (\cref{fig:reliability,fig:darcy-certify}). The held-out truths are synthetic. Details of the operators, training, and scoring are provided in the appendix.

\subsection{A two-dimensional example}
\label{ss:toy}

Here a two-component truth is read through $\bA=[1,0]$ and reconstructed under the truth itself or under a single Gaussian regularizer (\cref{app:toy}), yielding all four archives (\cref{fig:toy-archives}). Under the truth the posterior-sample archive returns the truth, since $\loopop[\truth]=\truth$ by \eqref{eq:emstep}; under $\reg$ its blind conditional is $\reg$'s (Proposition~\ref{p:blind}), and the single-best archive has zero blind width under either (Theorem~\ref{p:map}). Priors fit to the regularizer's two archives and to the truth report the blind widths the theory predicts for both a mixture fit and a coupling flow (\cref{fig:toy-money}); on the single-best archive that width is the fit's own smoothing floor (\cref{r:map-lg}).

\begin{figure}[!tb]
\centering
\includegraphics{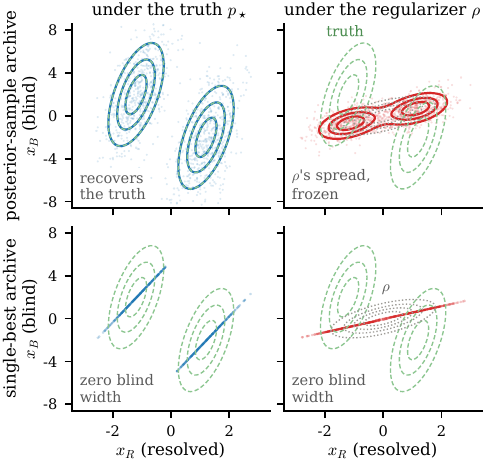}
\caption{The four archives, under the truth (blue) or $\rho$ (red). Dashed: truth; dotted: $\rho$; solid: exact posterior-sample laws.}
\label{fig:toy-archives}
\end{figure}

\begin{figure}[!tb]
\centering
\includegraphics{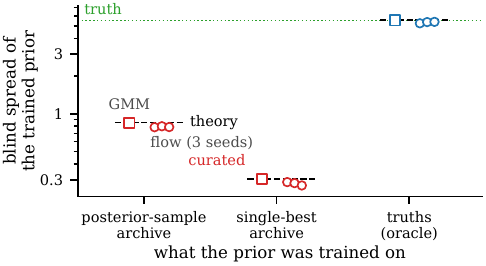}
\caption{Blind spread of priors fit to $\rho$'s two archives and to the truth; dashes: the exact dequantized widths.}
\label{fig:toy-money}
\end{figure}

\subsection{Coverage on the two subspaces}
\label{ss:reliability}

On each operator we train a \textcolor{madcol}{curated} prior on a single-best archive and an \textcolor{emcol}{oracle} control on truths, matched in architecture and budget. Both are diffusion priors on the seismic operator and coupling flows on the groundwater one. The protocol verifies capacity, since a prior too weak to fit its archive would mask the inherited assumption with spread of its own (appendix). The subspace split is the resolvability statement's own cut---on the seismic operator a low-illumination near-null read as a kernel, its cut set relative to the resolved directions' illumination (appendix), on the groundwater operator the base-state Jacobian's kernel together with its sub-noise tail. A truth-free source calibration makes the two seismic priors comparable (appendix).

Figure~\ref{fig:reliability} displays the reliability curves. On the groundwater operator the \textcolor{madcol}{curated} flow's preconditioned Crank--Nicolson (pCN) posterior substantially under-covers the blind subspace at every level, while the \textcolor{emcol}{oracle} tracks nominal and the two are close on the resolved subspace. On the seismic operator the shortfall is smaller and appears from the middle levels upward, beyond the bootstrap intervals, with the \textcolor{emcol}{oracle} too wide rather than under-covering. The pattern---severe at an exact kernel, mild on a near-null the data barely reach---is the one the near-null forms predict. The low-level resolved distortion the seismic panels share is due to the source calibration, which matches energies rather than quantiles; common to both priors, it does not alter the blind comparison.

\begin{figure}[!tb]
\centering
\includegraphics[width=\columnwidth]{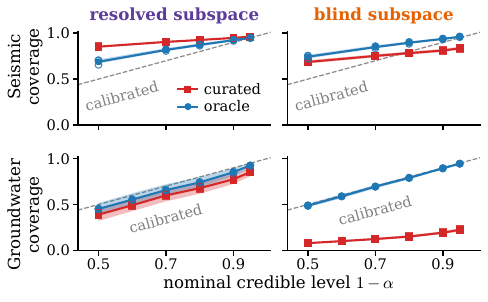}
\caption{Coverage against nominal: seismic prior samples (top), groundwater pCN posteriors (bottom). Open markers: training seeds; bands: bootstrap.}
\label{fig:reliability}
\end{figure}

\begin{figure}[!tb]
\centering
\includegraphics[width=\columnwidth]{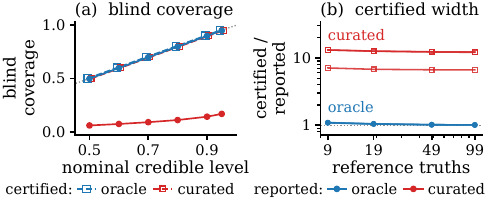}
\caption{The certified band on the groundwater kernel: (a)~blind coverage; (b)~certified over reported half-width against the reference count.}
\label{fig:darcy-certify}
\end{figure}

\subsection{The within-prior gap}
\label{ss:gap}

Neither prior lies on the diagonal at every level, so we subtract each prior's resolved coverage from its own blind coverage, which removes any error both subspaces share, and compare the two gaps.

Figure~\ref{fig:gap} shows the result. On both operators the \textcolor{emcol}{oracle}'s gap is flat and close to zero at every level, while the \textcolor{madcol}{curated} prior's lies well below it, with a scatter over training seeds far smaller than the separation. On the seismic operator the legacy archive scored the same way is close to the prior trained on it, so the prior reproduces the archive's own gap rather than merely degrading. Where the data reach, the \textcolor{madcol}{curated} prior has a spread as near the truth's as the control does, and the two differ substantially only where the data cannot reach. On a single groundwater observation the two posteriors differ in the same way (\cref{fig:darcy-posterior}). A subspace of the blind subspace's own dimension, randomly drawn from the probed seismic directions rather than sorted by illumination, leaves the two priors indistinguishable, so what separates them is the operator's blind subspace and not the act of splitting the space in two.

\begin{figure*}[!tb]
\centering
\includegraphics[width=1.00\textwidth]{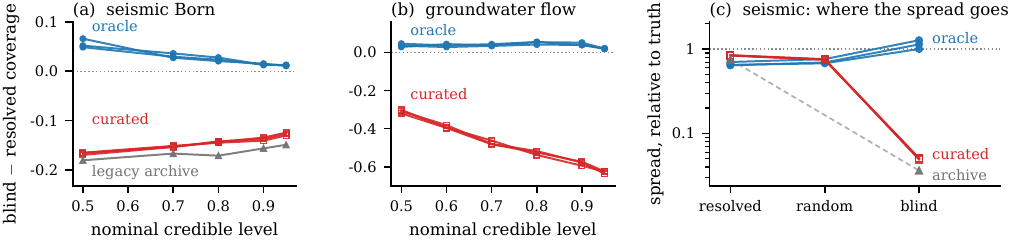}
\caption{(a,b)~each prior's blind-minus-resolved coverage gap, one line per training seed; the legacy archive in gray (panel~a). (c)~spread relative to the truth's; middle column: the random-subspace control; the archive at the two ends alone (dashed). Colors as in \cref{fig:teaser}.}
\label{fig:gap}
\end{figure*}

\begin{figure*}[!tb]
\centering
\includegraphics[width=0.95\textwidth]{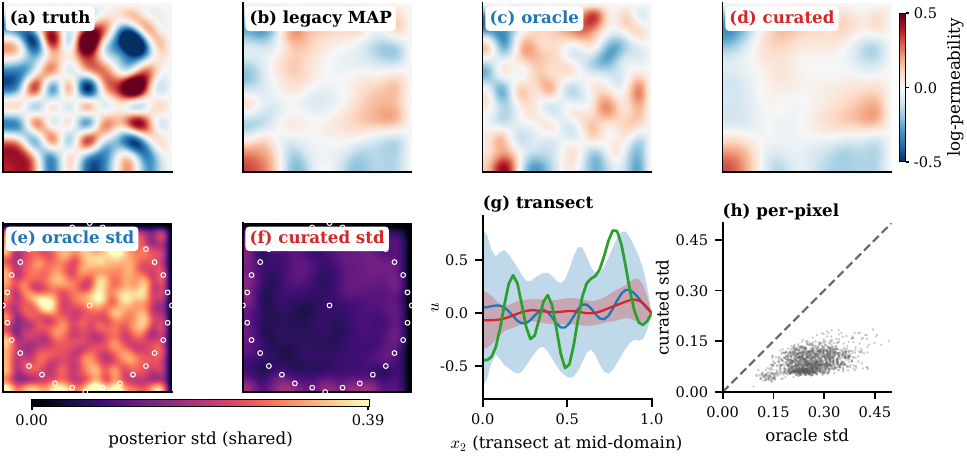}
\caption{Groundwater posteriors by pCN in the flow latent. (a)~truth; (b)~legacy MAP; (c,d)~posterior means and (e,f)~spreads under the \textcolor{emcol}{oracle} and \textcolor{madcol}{curated} priors (circles: sensors); (g)~mid-domain transect, central-90\% bands; (h)~pixelwise spread pairs, all below the diagonal.}
\label{fig:darcy-posterior}
\end{figure*}

\subsection{The band in practice}
\label{ss:certify}

The certified band (Proposition~\ref{p:certify}) is run on the groundwater operator, over its exact kernel, against references held out of every archive. In \cref{fig:darcy-certify} the band's coverage tracks the nominal level for both the \textcolor{madcol}{curated} and the \textcolor{emcol}{oracle} prior, while the band a prior reports tracks it only for the latter. The cost appears in the width, where the \textcolor{emcol}{oracle} requires almost no correction and the \textcolor{madcol}{curated} prior a large one. The ratio is stable beyond the handful of references the audit calculation calls for, and the seed that fits its archive worst requires the least correction, the ordering the mechanism predicts. The cutoff sweep, the added-borehole repair with its placebo and noise-sweep controls, the sampler's leakage check, and the coverage table are in the appendix.

\section{Discussion}
\label{sec:discussion}

The main limitation of our experiments is that their truths are synthetic, since evaluating coverage requires a known truth, which field data rarely provide. Our theory is exact at a kernel and, for the EM step, in the population limit, so what the trained priors report on the seismic near-null and the linearized groundwater kernel is an empirical reading rather than a guarantee. A remedy that moves the blind conditional is a different type of measurement, e.g., a well log, which reads the subsurface directly. Curating jointly against this measurement de-freezes the blind directions it resolves (Corollary~\ref{c:augment} and Appendix~\ref{app:additional}). While public migrated sections with well logs~\citep{bhar2026openseisml} may enable training generative priors on field data, our analysis does not address whether a prior trained with sparse wells is calibrated.

\section{Conclusion}
\label{sec:conclusion}

Where truths are scarce, the prior that regularizes an inverse problem is trained on an older method's outputs, and the uncertainty it reports inherits that method's assumption on the operator's blind subspace. The prior is one EM step from the regularizer where the data reach, frozen where they do not, and collapsed to zero width under the single-best archives practice keeps. Nothing in practice detects this, and remedies that stay within the observations---retraining, archiving samples, pooling methods, or agreement with a Markov chain Monte Carlo benchmark sharing the prior~\citep{zhang2020variational,li2026stein}---leave the blind conditional unchanged. The resolvability statement identifies the inherited directions from the operator alone, before any prior is trained, and the certified band turns a handful of reference truths into an interval whose coverage holds regardless of the prior. Amortized inference, which learns the whole conditional with no explicit prior, raises the same inheritance question without the object identified here.

\clearpage
\subsection*{AI use statement}
Generative AI tools were used to aid in polishing the writing of this paper, to assist in providing Lean code for machine-checking its mathematical proofs, and to partially implement code. All AI-assisted work was reviewed by the authors, who take responsibility for the final content of this work, including text, claims, and artifacts produced with the aid of generative AI.

\bibliographystyle{apalike}
\bibliography{refs}

\clearpage
\appendix
\thispagestyle{empty}
\onecolumn
\aistatstitle{Supplementary Materials}

\section{Full statements and proofs}
\label{app:proofs}

The development accompanying the paper machine-checks every statement here, and it records the correspondence statement by statement. Proved in the paper alone are the closing asymptotic approximation in \eqref{eq:auditk}, the general-mollifier form of the smoothing floor~(\cref{r:map-lg}), the no-better-score aside of \cref{p:certify}, and two classical steps internal to proofs---the Gaussian sufficiency reduction in the proof of \cref{c:augment} and the descent-and-compactness argument in \cref{p:recover}(ii). Dictionary items also remain on paper: the spectral bands stated variationally read as extreme eigenvalues, the identification of the development's scenario with the matrix notation here, and \cref{p:mapnearnull}(i)'s solved quadratic identity, a one-line consequence of the checked stationarity input; \cref{c:invariant} is obtained from the two checked parts of \cref{p:recover}.

\subsection{Full statements of the results abridged or referenced in the body}

\begin{theorem}[Theorem~\ref{p:map}, full statement: single-best curation collapses the blind-fiber uncertainty \emph{(any forward operator with a blind subspace, any factoring noise, any variational estimator)}]\label{p:map-full}
Under a single-best archive let it store any variational reconstruction
\begin{equation}
\hat\bx(\by)\in\arg\min_{\bx}\ D\big(\by,\bF(\bx)\big)+R(\bx),
\label{eq:varest-full}
\end{equation}
a data-fidelity term plus a penalty---the Bayesian $\mathrm{MAP}_\reg(\by)\in\arg\max_\bx p(\by\mid\bx)\,\reg(\bx)$ among its instances (on function space the maximizer is read through small-ball probabilities, the Onsager--Machlup definition of the MAP estimator~\citep{dashti2013map,helin2015map}, the density form here being its finite-dimensional reading), as are the penalized reconstructions of \cref{sec:experiments} and, with an improper penalty, total-variation reconstruction. Assume a minimizer exists (the measurable selection where it is not unique, and the comparison $C^{\mathrm{sb}}\le C^{\mathrm{ps}}$ of the blind coverage a single-best archive attains against a posterior-sample one's, are handled in \cref{app:proofs}); we write $C$ for blind coverage throughout, superscribed $\mathrm{sb}$ for the single-best archive and $\mathrm{ps}$ for the posterior-sample one. Then $\hat\bx(\by)$ is a \emph{deterministic} function of $\by$, so the curated ensemble contains none of the within-$\by$ posterior spread. On a blind fiber---the leaf $\bx+\blind$ of the operator's blind subspace $\blind$, with $\bF(\bx+\bm n)=\bF(\bx)$ and $\blind=\ker\bA$ for a linear operator---the data term is flat, so the blind component of the reconstruction is the penalty's conditional minimizer given the resolved coordinates---a function of those coordinates alone, never of the measurement; in the split coordinates,
\begin{equation}
\hat\bx_B(\bx_R)\in\arg\min_{\bx_B}R(\bx_R,\bx_B),\qquad C^{\mathrm{sb}}=0,
\label{eq:mapcollapse-full}
\end{equation}
a single point with zero within-fiber spread. The conditional credible interval on a blind functional $\theta=\bm v^\top\bx$ ($\bm v$ in the blind subspace) therefore has zero width and never covers a truth whose fiber-conditional law of $\theta$ is nonatomic (a truth with a density along the fiber, in particular).
\end{theorem}

\begin{proposition}[A closed-form coverage shortfall \emph{(linear--Gaussian)}]\label{p:cover}
Take a blind direction; write $s_\reg,s_\star$ for the curated prior's and the truth's blind-fiber \emph{conditional} spreads along it~(\cref{p:cover-na}), and $\delta=(m_\reg-m_\star)/s_\star$ for the standardized gap between their conditional means. The freeze fixes the \emph{entire} conditional (mean as well as spread) to the regularizer's~(\cref{p:blind}), so the credible interval the prior reports at level $1-\alpha$ ($z=z_{1-\alpha/2}$ the standard-normal quantile) is centered at $m_\reg$ and covers the truth at
\begin{equation}
C \;=\; \Phi\!\big(\delta+z\,s_\reg/s_\star\big)-\Phi\!\big(\delta-z\,s_\reg/s_\star\big).
\label{eq:coverage-gen}
\end{equation}
A correctly centered regularizer ($\delta=0$) gives the width-only form
\begin{equation}
C \;=\; 2\,\Phi\!\left(z\,\frac{s_\reg}{s_\star}\right) - 1 ,
\end{equation}
overconfident ($C<1-\alpha$) exactly when the inherited assumption is tighter than the truth, $s_\reg<s_\star$; a frozen mean error $\delta\neq0$ only exacerbates it.
\end{proposition}

\begin{corollary}[The shortfall is not identifiable from the measurements \emph{(any likelihood that factors through the forward operator)}]\label{c:nonident}
For any noise model that depends on the unknown only through the operator, $p(\by\mid\bx)=p(\by\mid\bF(\bx))$, the data marginal $\truth(\by)$ depends on the truth only through its resolved marginal (for an operator with blind subspace $\blind$, the law of its projection along $\blind$), so the truth's \emph{entire} blind-fiber conditional law $\truth(\bx_B\mid\bx_R)$ is not a functional of the data marginal, so neither its blind spread nor its blind mean is, and neither therefore is any comparison between them and the prior's~(\cref{p:cover}). No statistic computed from the measurements---no goodness-of-fit test, no held-out check against the recorded data---can detect or correct the inherited overconfidence (proved in \cref{app:proofs}; it relies only on the likelihood's fiber-constancy, witnessed for a linear operator by $\bA\bN=\mathbf 0$). In the linear--Gaussian case the witness is explicit: two truths---two candidate laws for the unknown---differing only on the blind subspace,
\begin{equation}
\Sigma_\star\ \longmapsto\ \Sigma_\star+\bN\bm W\bN^\top+\bN\bm C\bP_R+\bP_R\bm C^\top\bN^\top,\qquad \mu_\star\ \longmapsto\ \mu_\star+\bN\bm w
\label{eq:witness}
\end{equation}
($\bm W\succeq0$; $\bm C$ small enough that the perturbed covariance remains positive definite), varying the blind conditional's spread, mean, and dependence on the readout arbitrarily, induce the identical measurement law. Under single-best curation the failure is sharper still, because $\mathrm{MAP}_\reg$ is a truth-independent deterministic map of $\by$, the two truths yield identically distributed archives and train the identical learned prior; for a fixed set of measurements the archive is unchanged.
\end{corollary}

\begin{remark}[Self-consistency checks are neutral by construction]\label{r:selfcheck}
Simulation-based calibration generates a test truth from the model's own prior, simulates a measurement from it, samples the posterior and ranks
the truth among the samples, so it reads nothing from the recorded data and \cref{c:nonident} does not reach it. A separate argument does:
whenever the sampler targets the prior that generated the test truth, that truth and the posterior samples are exchangeable, so the rank of any
one-dimensional functional with almost surely tie-free samples---any smoothing scale ensures this---is uniform whatever the prior~\citep{talts2018validating}. This check scores a prior's consistency with its own posterior, never its fidelity to the truth, and its law is the same for the curated prior as for a truth-trained one, its power against the freeze equal to its size. Every diagnostic available in practice is a functional of the recorded data or of the prior and likelihood alone; \cref{c:nonident} closes the first family and exchangeability the second.
\end{remark}

\begin{corollary}[The blind conditional is a loop invariant, and labels the limit \emph{(invariance general; completeness linear--Gaussian)}]\label{c:invariant}
For any likelihood depending on the unknown only through the operator, any regularizer, and any law by which each round's measurements are averaged, the blind-fiber conditional is unchanged by every iterate of the loop---so it is an invariant, and no amount of curation against fresh real data alters it~(\cref{p:recover}(i)). In the linear--Gaussian case, under that proposition's hypotheses, it is moreover \emph{complete} for the limit, in that the orbit converges to the law carrying the truth's resolved marginal together with the initial regularizer's blind conditional, so two regularizers reach the same limit if and only if they share a blind-fiber conditional, and nothing else about the starting point remains. Consequently the loop recovers the full truth from every \emph{full-support} regularizer---the hypothesis is on the regularizer, not only on the operator---if and only if the operator has trivial kernel.
\end{corollary}

\begin{proposition}[How many reference truths an audit needs \emph{(Gaussian fibers, noiseless references)}]\label{p:audit}
Let $\bN$ span the blind subspace, of dimension $b$, and let the curated prior provide its blind-fiber conditional mean map $\bm m_\reg(\bx_R)$ and conditional covariance $\Sigma_{B\mid R}^{\reg}$. Suppose $k$ references give the truth exactly, drawn independently with a Gaussian fiber conditional---each at its own resolved readout, which the Schur complement $\Sigma_{B\mid R}^{\reg}$ does not depend on---and form the standardized residuals
\begin{equation}
\bm u_i \;=\; (\Sigma_{B\mid R}^{\reg})^{-1/2}\bN^\top\big(\bx_i-\bm m_\reg(\bx_{R,i})\big)\in\R^{b}.
\label{eq:auditstat}
\end{equation}
If the prior's blind report is correct then $\bm u_i\sim\mathcal N(\bm 0,\bm I_b)$ independently, so $T=\sum_i\|\bm u_i\|^2\sim\chi^2_{kb}$. Against a truth whose blind spread exceeds the reported one by a factor $r$ throughout the subspace, the size-$\alpha_{\mathrm{t}}$ test rejecting at $\chi^2_{kb,1-\alpha_{\mathrm{t}}}$ has power
\begin{equation}
1-\beta \;=\; \Pr\!\Big(\chi^2_{kb} > \chi^2_{kb,\,1-\alpha_{\mathrm{t}}}\big/ r^2\Big),
\qquad\text{so}\qquad
k \;\approx\; 1+\frac{\big(z_{1-\alpha_{\mathrm{t}}}+z_{1-\beta}\big)^2}{2\,b\,\log^2 r}.
\label{eq:auditk}
\end{equation}
The leading $1$ is a rounding-up of the asymptotic solution, not part of it. The reference count falls in inverse proportion to the blind dimension, so the larger the subspace the data cannot check, the cheaper it is to check it~(\cref{tab:audit})---against this alternative, which grows with $b$.
\end{proposition}

\begin{table}[!hb]
\centering
\caption{The smallest reference count at which the power in \eqref{eq:auditk} reaches $0.8$ at level $0.05$, by blind dimension $b$.}
\label{tab:audit}
\begin{tabular}{lllll}
\toprule
spread ratio $r$ & $b=1$ & $b=6$ & $b=69$ & $b=256$ \\
\midrule
$2$    & $7$    & $2$   & $1$  & $1$ \\
$1.5$  & $18$   & $3$   & $1$  & $1$ \\
$1.2$  & $91$   & $16$  & $2$  & $1$ \\
$1.05$ & $1286$ & $215$ & $19$ & $6$ \\
\bottomrule
\end{tabular}
\end{table}

\begin{remark}[Observations and references compared, and why only one remains finite]\label{r:exchange}
The same misreport can be pursued either way, and the two requirements behave oppositely. Write the inflation of the blind spread along $\bm v$ as $\Sigma_\star\mapsto\Sigma_\star+w\,\bm v\bm v^\top$, so that the spread ratio of \cref{p:audit} is $r=\sqrt{1+w/s_\star^2}$. Obtaining the discrimination from observations requires at least the crossover of \cref{p:nearnull},
\begin{equation}
n_\star \;\gtrsim\; \Big(\frac{\lambda_{\min}(\Gamma)}{w\,\|\bA\bm v\|^{2}}\Big)^{2},
\end{equation}
which diverges as the inverse square of the illumination $\|\bA\bm v\|^2$ as the direction becomes blind and is infinite on a genuine kernel. Obtaining it from references requires the one-dimensional $k$ of \eqref{eq:auditk}, since a rank-one misreport is the case that dilutes in the pool. Crucially, this count does not involve $\bA$ at all, and it is the same whether the direction is poorly illuminated or exactly blind. The two therefore do not trade off at a fixed rate, since precisely on the directions where no number of observations suffices, the number of references needed is unchanged and small.
\end{remark}

\begin{corollary}[An added measurement de-freezes exactly the directions it resolves \emph{(de-freezing general; recovery linear--Gaussian)}]\label{c:augment}
Let $\bm B_1$ have orthonormal columns spanning a subspace $\blind_1\subseteq\ker\bA$ of the blind subspace, observed through a new measurement $\by'=\bm B_1^\top\bx+\bm\varepsilon'$ whose noise law depends on $\bx$ only through $\bm B_1^\top\bx$ and is independent of the observation noise given $\bx$. Curating \emph{jointly} against $(\by,\by')$ is the map $\loopop$ of \eqref{eq:curated} for the stacked operator $[\bA;\bm B_1^\top]$, whose kernel is $\ker\bA\cap\blind_1^\perp$; since $\ker\bA=\blind_1\oplus(\ker\bA\cap\blind_1^\perp)$, exactly $\blind_1$ leaves the frozen set, and nothing else. In the linear--Gaussian case, provided $\Sigma_\star\succ0$ and the regularizer carries spread on the enlarged resolved space---$\bA_r\Sigma_\reg\bA_r^\top\succ0$ for a full-row-rank $\bA_r$ spanning $\mathrm{row}(\tilde\bA)$, $\tilde\bA=[\bA;\bm B_1^\top]$, which a full-support regularizer ensures---the repeated loop moreover corrects the prior on $\mathrm{row}(\bA)\oplus\blind_1$: de-freezing is general, recovery is linear--Gaussian.
\end{corollary}

\subsection{Proofs}

\begin{proof}[Proof of \Cref{p:law}]
By Bayes' rule $\reg(\bx\mid\by)=p(\by\mid\bx)\reg(\bx)/\reg(\by)$ with $\reg(\by)=\int p(\by\mid\bx)\reg(\bx)\,d\bx$, so
\begin{equation}
\loopop[\reg](\bx)=\int \reg(\bx\mid\by)\,\truth(\by)\,d\by
=\reg(\bx)\int \frac{p(\by\mid\bx)}{\reg(\by)}\,\truth(\by)\,d\by
=\reg(\bx)\,\E_{\by\sim\truth}\!\Big[\tfrac{p(\by\mid\bx)}{\reg(\by)}\Big].
\label{eq:emstep-app}
\end{equation}
For the ascent, write $g=\loopop[\reg]/\reg$ for the reweight in \eqref{eq:emstep-app}, so that $\loopop[\reg](\by)=\reg(\by)\,\E_{\reg(\bx\mid\by)}[g]$. Jensen's inequality on the logarithm, then \eqref{eq:emstep-app} once more, gives
\begin{align*}
&\KL\!\big(\truth(\by)\,\big\|\,\reg(\by)\big)-\KL\!\big(\truth(\by)\,\big\|\,\loopop[\reg](\by)\big)\\
&\quad=\E_{\by\sim\truth}\log\E_{\reg(\bx\mid\by)}[g]
\ \ge\ \E_{\by\sim\truth}\E_{\reg(\bx\mid\by)}[\log g]\\
&\quad=\KL\!\big(\loopop[\reg]\,\big\|\,\reg\big)\ \ge\ 0,
\end{align*}
each term finite by hypothesis, so the step never decreases the data-marginal likelihood.
\end{proof}

\begin{proof}[Proof of \Cref{p:blind}]
Write $\bx=(\bx_R,\bx_B)$ in the $\mathrm{row}(\bA)\oplus\ker\bA$ split. Since the likelihood depends on $\bx$ only through $\bA\bx=\bA\bx_R$ (the Gaussian $\mathcal N(\bA\bx,\Gamma)$ of \eqref{eq:likelihood} in particular), the reweight in \eqref{eq:emstep-app},
\begin{equation}
g(\bx):=\E_{\by\sim\truth}\!\big[p(\by\mid\bx)/\reg(\by)\big]=g(\bx_R),
\end{equation}
is fiber-constant. Hence the curated blind-fiber conditional is
\begin{equation}
q_\reg(\bx_B\mid\bx_R)=\frac{q_\reg(\bx_R,\bx_B)}{\int q_\reg(\bx_R,\bx_B')\,d\bx_B'}
=\frac{\reg(\bx_R,\bx_B)\,g(\bx_R)}{g(\bx_R)\int \reg(\bx_R,\bx_B')\,d\bx_B'}
=\reg(\bx_B\mid\bx_R).
\label{eq:freeze-cond}
\end{equation}
The same computation covers any blind subspace $\blind$: $\bF(\bx)=\bF(\bx_R)$ makes the reweight a function of $\bx_R$ alone, so \eqref{eq:freeze-cond} applies directly. In the Gaussian case the conditional $\bx_B\mid\bx_R$ has precision $\bN^\top\Sigma^{-1}\bN$, the blind block of the precision matrix, so \eqref{eq:freeze-cond} for $q_\reg$ and $\reg$ reads block by block as \eqref{eq:freeze}.
\end{proof}

\begin{proof}[Proof of \Cref{p:map}]
We prove the full form (\cref{p:map-full}).
\emph{Collapse, general.} The data term $D(\by,\bF(\bx))$ of \eqref{eq:varest} depends on $\bx$ only through $\bF(\bx)$, hence is constant on every fiber of the forward operator; this covers every noise kernel $p(\by\mid\bx)=k(\by,\bF(\bx))$ through $D=-\log k$, in particular any additive noise $\nu(\by-\bF(\bx))$, and every penalty $R$, proper or improper. Comparing candidates within one fiber cancels the data term: every minimizer minimizes the penalty over its own fiber, whatever $\bF$.

For a linear operator the fibers are the leaves $\bx+\ker\bA$; decompose $\bx=(\bx_R,\bx_B)$, $\bx_R=\bP_R\bx$, $\bx_B=\bN\bN^\top\bx$, and profile the objective over the fiber,
\begin{equation}
\hat\bx(\by)\in\arg\min_{\bx_R,\bx_B}\ D\big(\by,\bF(\bx_R)\big)+R(\bx_R,\bx_B),\qquad
\hat\bx_B(\bx_R)\in\arg\min_{\bx_B}R(\bx_R,\bx_B).
\end{equation} Where the $\arg\min$ is a set (the fiber sections of a total-variation penalty need not be strictly convex) the profiled selection makes $\hat\bx_B$ a single-valued function of the resolved coordinates (the convention the Dirac fiber conditional below uses, since a measurement-dependent tie-break could spread the conditional over the fiber's $\arg\min$ set), while the zero-coverage conclusion holds for every measurable selection, and a measurable selection exists by the Kuratowski--Ryll-Nardzewski theorem~\citep{kuratowski1965selectors}.

For the Bayesian MAP, $R=-\log\reg$ and $\reg(\bx_B\mid\bx_R)\propto\reg(\bx_R,\bx_B)$ at fixed $\bx_R$, so the conditional minimizer is the regularizer's conditional mode $\hat\bx_B(\bx_R)\in\arg\max_{\bx_B}\reg(\bx_B\mid\bx_R)$, and the outer minimizer is a deterministic function of $\by$. Hence, conditional on the resolved readout, the curated blind law is the Dirac $\delta_{\hat\bx_B(\bx_R)}$: the conditional credible interval on any $\theta=\bm v^\top\bx$, $\bm v\in\ker\bA$, has zero width, and for a truth whose fiber-conditional law of $\theta$ is nonatomic (an absolutely continuous fiber conditional suffices, for every $\bm v\neq\bm 0$) its coverage $C^{\mathrm{sb}}=\Pr(\theta_\star=\bm v^\top\hat\bx_B)=0$, with $C^{\mathrm{sb}}\le C^{\mathrm{ps}}$ whenever the penalty is a proper prior's, so that a posterior-sample archive exists; an improper penalty such as a total-variation seminorm has no normalizable prior and no posterior-sample counterpart, and the collapse still holds. No Gaussian or alignment hypothesis is used.

\emph{Reach.} A subspace invariance $\bF(\bx+\bm n)=\bF(\bx)$, $\bm n\in\blind$, is the case $\bF=\bF\circ\bP_{\blind^\perp}$: the data term is a function of the resolved coordinate $\bx_R=\bP_{\blind^\perp}\bx$ alone, the profiling above applies directly, and \eqref{eq:mapcollapse} holds with $\ker\bA$ replaced by $\blind$.

The likelihood is a function of $\bx_R$, so $\by$ and $\bx_B$ are conditionally independent given $\bx_R$, while the archive's report is a function of $\by$ alone. When the truth's conditional law of $\theta$ given the resolved coordinates is nonatomic, the report is conditionally a single point of probability zero, and the tower property gives
\begin{equation*}
C^{\mathrm{sb}}=\E\big[\Pr\big(\theta_\star=c(\by)\mid\bx_{R\star},\by\big)\big]=0
\end{equation*}
for every measurable reconstruction rule $c$; the variational structure determines only where the point lies. The same cancellation holds along the level sets of any operator (e.g., the global phase of phase retrieval), and the $\arg\min$ form is not essential: an archive of posterior means, the denoised outputs real pipelines also store, is a deterministic function of the measurement, so the display above already gives it zero blind coverage; for a linear operator under Gaussian noise, with a prior of finite first moment holding no resolved direction a.e.\ constant, it is again a deterministic function of its own resolved component, with the same zero-width fiber report.
\end{proof}

\begin{proof}[Proof of \Cref{p:cover}]
Fix $\bm v\in\ker\bA$, $\|\bm v\|=1$, $\theta=\bm v^\top\bx$, and partition $\Sigma_\reg$ in the $\mathrm{row}(\bA)\oplus\ker\bA$ split,
\begin{equation}
\Sigma_\reg=\begin{bmatrix}\Sigma_{RR}&\Sigma_{RB}\\[2pt]\Sigma_{BR}&\Sigma_{BB}\end{bmatrix}.
\end{equation}
The Gaussian conditioning formula gives the regularizer's blind-fiber conditional at readout $\bx_R$ as $\mathcal N\big(\bm\mu_{B\mid R}(\bx_R),\Sigma_{B\mid R}\big)$ with
\begin{equation}
\bm\mu_{B\mid R}(\bx_R)=\mu_B+\Sigma_{BR}\Sigma_{RR}^{-1}(\bx_R-\mu_R),\qquad
\Sigma_{B\mid R}=\Sigma_{BB}-\Sigma_{BR}\Sigma_{RR}^{-1}\Sigma_{RB},
\label{eq:schur}
\end{equation}
the Schur complement; projecting on $\bm v$ gives $m_\reg=\bm v^\top\bm\mu_{B\mid R}$, $s_\reg^2=\bm v^\top\Sigma_{B\mid R}\bm v$, and identically $m_\star,s_\star$ for the truth. By the freeze \eqref{eq:freeze-cond} the curated conditional of $\theta$ is $\mathcal N(m_\reg,s_\reg^2)$; scored against the true $\theta\sim\mathcal N(m_\star,s_\star^2)$, the reported interval $m_\reg\pm z s_\reg$ covers with probability
\begin{equation}
C=\Pr_{\mathcal N(m_\star,s_\star^2)}\!\big(|\theta-m_\reg|\le z s_\reg\big)
=\Phi\big(\delta+z s_\reg/s_\star\big)-\Phi\big(\delta-z s_\reg/s_\star\big),\quad \delta=\tfrac{m_\reg-m_\star}{s_\star},
\end{equation}
which is \eqref{eq:coverage-gen}; $\delta=0$ gives \eqref{eq:coverage}, and $C<2\Phi(z)-1=1-\alpha$ exactly when $s_\reg<s_\star$.

Differentiating in the mean error,
\begin{equation}
\frac{\partial C}{\partial\delta}=\phi(\delta+z s_\reg/s_\star)-\phi(\delta-z s_\reg/s_\star)<0\quad(\delta>0),
\end{equation}
a frozen-mean error only reduces coverage. The derivation used the blind-fiber \emph{conditional} spreads \eqref{eq:schur} and no block-diagonal assumption, i.e., the alignment-free \cref{p:cover-na}; under alignment $\Sigma_{RB}=0$ the conditional equals the marginal, and \eqref{eq:coverage} follows.
\end{proof}

\begin{proof}[Proof of \Cref{c:nonident}]
\emph{General case.} If $p(\by\mid\bx)=p(\by\mid\bF(\bx))$ and the operator has blind subspace $\blind$, then
\begin{equation}
\bF(\bx)=\bF(\bP_{\blind^\perp}\bx),
\end{equation}
so the data marginal is a functional of the resolved marginal alone: two truth distributions sharing it but carrying different blind-fiber conditionals induce the identical measurement law, the \emph{entire} blind conditional is non-identifiable. For a linear operator $\bA=\bA\bP_R$, so $\truth(\by)$ is a functional of the resolved marginal $\bP_{R\#}\truth$ alone.

\emph{Linear--Gaussian witness.} Here the data marginal is
\begin{equation}
\truth(\by)=\int p(\by\mid\bx)\,\truth(\bx)\,d\bx=\mathcal N\big(\bA\mu_\star,\ \bA\Sigma_\star\bA^\top+\Gamma\big)=\mathcal N(\bA\mu_\star,\bS_y),
\end{equation}
a function of $\bA\mu_\star$ and $\bA\Sigma_\star\bA^\top$ only. Since $\bA\bN=\mathbf 0$ and $\bA\bP_R=\bA$, every added block in \eqref{eq:witness} is annihilated,
\begin{equation}
\bA(\bN\bm W\bN^\top)\bA^\top=\mathbf 0,\qquad
\bA(\bN\bm C\bP_R+\bP_R\bm C^\top\bN^\top)\bA^\top=\mathbf 0,\qquad
\bA\bN\bm w=\mathbf 0,
\end{equation}
so $\bA\Sigma_\star\bA^\top$ and $\bA\mu_\star$, hence $\truth(\by)$, are unchanged, while $s_\star$, $\delta$, and the truth's resolved--blind cross-covariance are all free to vary; none of them is therefore a function of the measurement law. For a single-best archive, $\mathrm{MAP}_\reg$ is a fixed function of $\by$ and the two truths share the measurement law, so the archives are identically distributed, and pathwise identical for a fixed set of measurements.
\end{proof}

\begin{proof}[Proof of \Cref{c:honest}]
The archive is a function of the record and the pipeline's own randomness, so $\hat I$ is measurable with respect to information whose law is constant on $\mathcal P$; the coverage bound is the density bound integrated over that law, and the supremum of the conditional spread over $\mathcal P$ is infinite by the witness.
\end{proof}

\begin{proof}[Proof of \Cref{c:invariant}]
Both halves are \cref{p:recover}, whose part~(i) gives the invariance for an arbitrary prior and averaging law and whose part~(ii) gives the resolved convergence; the limit therefore carries the truth's resolved marginal and, by~(i), the regularizer's blind conditional. The final equivalence is that proposition's closing statement.
\end{proof}

\begin{proof}[Proof of \Cref{p:audit}]
The residuals are standardized by quantities the prior provides, so under the null they are standard normal without anything estimated, and $T$ is a sum of $kb$ independent squares. Inflating the conditional covariance to $r^2\Sigma_{B\mid R}^{\reg}$ scales each residual by $r$, giving $T\sim r^2\chi^2_{kb}$ and hence the stated power; the rejection probability increases in the inflation, so the test is level $\alpha$ over the composite null. For the sample size, $\log T$ is asymptotically normal with mean separation $2\log r$ and variance $2/(kb)$, and solving gives \eqref{eq:auditk}.
\end{proof}

\begin{proof}[Proof of \Cref{p:certify}]
Exchangeability of the references with the truth reported on makes the scores exchangeable, so the test score's rank among the $k+1$ is uniform (exactly uniform under tie-free scores), and the empirical-quantile band covers at the nominal level~\citep{vovk2005}; the score's dependence on the prior is immaterial, since holding the references out of the archive allows the argument to condition on it. The band is an interval because $\eta(\theta)$ is quasi-convex in $\theta$. The conditional claim and the no-better-score aside that follows the proposition combine the fiber-constancy of the likelihood used in \cref{p:blind} with independence of the resolved and blind coordinates, which upgrades conditional independence given $\bx_R$ to independence; without it $\by$ still informs $\theta$ through the resolved coordinates.
\end{proof}

\begin{proof}[Proof of \Cref{c:augment}]
\emph{De-freezing, general.} Given $\bx$, the observation and the added measurement are independent, depending on $\bx$ only through $\bA\bx$ and $\bm B_1^\top\bx$, so the joint likelihood factors, $p(\by,\by'\mid\bx)=p(\by\mid\bA\bx)\,p'(\by'\mid\bm B_1^\top\bx)$, and depends on $\bx$ only through the stacked image $\tilde\bA\bx$, $\tilde\bA=[\bA;\bm B_1^\top]$.
Joint curation, i.e., averaging the $\reg$-posterior over the true joint law of $(\by,\by')$, is therefore the map \eqref{eq:curated} for $\tilde\bA$, so the cancellation \eqref{eq:freeze-cond} applies directly along its fibers and the curated conditional is frozen at the regularizer's exactly on $\ker\tilde\bA=\ker\bA\cap\ker\bm B_1^\top=\ker\bA\cap\blind_1^\perp$.
Since $\blind_1\subseteq\ker\bA$, $\ker\bA=\blind_1\oplus(\ker\bA\cap\blind_1^\perp)$,
so the frozen set shrinks by exactly $\blind_1$; iterating the joint loop keeps every direction of $\ker\bA\cap\blind_1^\perp$ frozen for all rounds (\cref{p:blind}, applied at each round). Nothing here is Gaussian: any factoring noise on either measurement, any prior.

\emph{Recovery, linear--Gaussian.} With $\bm\varepsilon'\sim\mathcal N(0,\Gamma')$, $\Gamma'\succ0$, the stacked model is linear--Gaussian with operator $\tilde\bA$ and block-diagonal noise covariance. We do not argue that the stack inherits full row rank; it need not, and neither need $\bA$. Instead the sufficiency reduction below replaces $(\tilde\bA,\tilde\Gamma)$ by a full-row-rank pair $(\bA_r,\Gamma_r)$ spanning the same row space, leaving the loop unchanged, and \cref{p:recover}(ii) applies to that pair under $\Sigma_\star\succ0$ and $\bA_r\Sigma_\reg\bA_r^\top\succ0$: the iterated loop recovers the truth on $\mathrm{row}(\tilde\bA)=\mathrm{row}(\bA)\oplus\blind_1$, while the smaller kernel remains frozen.

The nondegeneracy asked of the regularizer is on the \emph{enlarged} resolved space, and does not follow from the corresponding condition for the observation alone. Along the recursion of \cref{p:recover}(ii) the data-space covariance satisfies $\mathrm{range}(\bm D_{t+1})\subseteq\mathrm{range}(\bm D_t)$, so the range can never grow: a direction that the added measurement newly resolves, but on which the regularizer has no spread to begin with, remains at zero variance for every round. Such a direction leaves the frozen set and is then never corrected: de-freezing without recovery. A full-support regularizer ensures the condition.

\emph{Several observations.} Joint curation against conditionally independent observations with likelihoods $p_j(\by_j\mid\bA_j\bx)$ factors the same way through the stack $[\bA_1;\dots;\bA_J]$, so the freeze holds exactly on the common kernel $\bigcap_j\ker\bA_j$. A redundant Gaussian stack lacks full row rank, but the sufficiency reduction leaves the loop unchanged. Write the stacked data and its full-row-rank factorization as
\[
\tilde\by=\tilde\bA\bx+\tilde{\bm\varepsilon},\quad\tilde{\bm\varepsilon}\sim\mathcal N(0,\tilde\Gamma),\qquad\tilde\bA=\bm L\bA_r,
\]
with $\bA_r$ of full row rank spanning the common row space and $\bm L$ of full column rank, and set
\[
\bm t=\Gamma_r\bm L^\top\tilde\Gamma^{-1}\tilde\by,\qquad\Gamma_r=(\bm L^\top\tilde\Gamma^{-1}\bm L)^{-1},\qquad\bm t\mid\bx\sim\mathcal N(\bA_r\bx,\Gamma_r).
\]
The $\bx$-dependent factor of the Gaussian likelihood,
\[
\exp\big(\bx^\top\bA_r^\top\Gamma_r^{-1}\bm t-\tfrac12\,\bx^\top\bA_r^\top\Gamma_r^{-1}\bA_r\bx\big),
\]
reaches $\tilde\by$ through $\bm t$ alone, and the $\bx$-free factor cancels between $p(\tilde\by\mid\bx)$ and $\reg(\tilde\by)$ in \eqref{eq:emstep-app}. Hence the loop for $(\tilde\bA,\tilde\Gamma)$ is the loop for the reduced full-row-rank pair $(\bA_r,\Gamma_r)$, and \cref{p:recover} applies to the latter.

\emph{Why the fusion must be joint.} Pooling an archive curated against $\bA_1$ with one curated against $\bA_2$ gives the readout-tilted mixture
\[
\frac{1}{K}\sum_k\reg\,g_k,\qquad g_k(\bx)=\E_{\by_k}\!\big[p_k(\by_k\mid\bx)/\reg(\by_k)\big],
\]
with $g_k$ the $k$-th observation's reweight. Each $g_k$ depends on $\bx$ only through $\bA_k\bx$, hence is constant along fibers of the \emph{common} kernel $\bigcap_j\ker\bA_j$, so the cancellation \eqref{eq:freeze-cond} passes to the mixture and its conditional there remains frozen at the regularizer's. Off the common kernel the pooled spread reports the observations' disagreement, centered on the average of what they report, not the joint correction.
\end{proof}

\section{Near-null forms and supporting results}
\label{app:nearnull}

Throughout, the measurements live in $\R^m$ ($\by\in\R^m$), $\bN$ has orthonormal columns spanning $\ker\bA$ (the blind subspace), and $\bP_R$ projects onto $\mathrm{row}(\bA)$ (the resolved subspace). In the linear--Gaussian case $\truth=\mathcal N(\mu_\star,\Sigma_\star)$, $\reg=\mathcal N(\mu_\reg,\Sigma_\reg)$, and $p(\by\mid\bx)=\mathcal N(\bA\bx,\Gamma)$, the posterior under $\reg$ is $\mathcal N(\mu_{\mathrm{post}}(\by),\Sigma_{\mathrm{post}})$ with
\[
\Sigma_{\mathrm{post}}=(\Sigma_\reg^{-1}+\bM)^{-1},\quad \bM=\bA^\top\Gamma^{-1}\bA,\quad
\mu_{\mathrm{post}}(\by)=\Sigma_{\mathrm{post}}(\Sigma_\reg^{-1}\mu_\reg+\bA^\top\Gamma^{-1}\by),
\]
and we write $\bG=\Sigma_{\mathrm{post}}\bA^\top\Gamma^{-1}$ and $\bS_y=\bA\Sigma_\star\bA^\top+\Gamma$ for the true data covariance.

\subsection{The nonlinear derivative (\cref{r:nonlinear})}

\begin{remark}[What remains for a nonlinear operator: a derivative statement]\label{r:nonlinear}
For a nonlinear forward map the resolved/blind split is the local one fixed at a model $\bx_0$ by the Jacobian $\bJ(\bx_0)$: a direction $\bm v$ is blind at $\bx_0$ when $\bJ(\bx_0)\bm v=0$. For additive noise with a $C^1$ density and differentiation under the expectation dominated, the curation's reweight in \eqref{eq:emstep} has vanishing first derivative along every blind direction at $\bx_0$, a derivative at the point, not a freeze on a neighborhood; away from $\bx_0$ the blind subspace rotates, which a genuine invariance removes.
\end{remark}

\begin{proof}[Proof of \cref{r:nonlinear}]
Along a blind direction of a nonlinear $\bF$ that is not an invariance, only the derivative statement of \cref{r:nonlinear} is claimed: with $p(\by\mid\bx)=\nu(\by-\bF(\bx))$, $\nu\in C^1$, and a dominating envelope allowing differentiation under the expectation, the chain rule gives, for $\bm v\in\ker\bJ(\bx_0)$,
\[
\frac{d}{dt}\,\E_{\by\sim\truth}\!\left[\frac{\nu(\by-\bF(\bx_0+t\bm v))}{\reg(\by)}\right]\bigg|_{t=0}
=-\,\E_{\by\sim\truth}\!\left[\frac{\nabla\nu(\by-\bF(\bx_0))^\top\bJ(\bx_0)\bm v}{\reg(\by)}\right]=0,
\]
since $\bJ(\bx_0)\bm v=\bm 0$.
\end{proof}

\subsection{Coverage without alignment (\cref{p:cover-na})}

\noindent The derivation in the proof of \cref{p:cover} used only that the two fiber conditionals at the fixed readout are Gaussian, and the sign of $\partial C/\partial\delta$ computed there bounds $C(\bx_R)$ by its $\delta=0$ value pointwise in the readout; averaging $\delta$ over the truth's Gaussian resolved marginal and applying $\E_Z[\Phi(a+bZ)]=\Phi(a/\sqrt{1+b^2})$ gives the aggregate. This proves:
\begin{proposition}[Blind-subspace coverage without alignment \emph{(linear--Gaussian)}]\label{p:cover-na}
Fix a blind direction $\bm v\in\ker\bA$; write $m_\reg(\bx_R),m_\star(\bx_R)$ for the curated prior's and the truth's blind-fiber conditional means at resolved readout $\bx_R$ and $s_\reg,s_\star$ for their conditional spreads, constant in the readout as in \eqref{eq:schur}, and assume only that the truth's fiber conditional is Gaussian (arbitrary $\Sigma_\reg,\Sigma_\star$ with no alignment in the linear--Gaussian case, and arbitrary, possibly nonlinear, conditional-mean maps beyond it). The reported band $m_\reg(\bx_R)\pm z\,s_\reg$ covers the truth's fiber conditional, with $\delta(\bx_R)=(m_\reg(\bx_R)-m_\star(\bx_R))/s_\star$, at
\[
C(\bx_R)=\Phi\big(\delta(\bx_R)+z\,s_\reg/s_\star\big)-\Phi\big(\delta(\bx_R)-z\,s_\reg/s_\star\big)\;\le\;2\,\Phi\!\big(z\,s_\reg/s_\star\big)-1,
\]
with equality exactly at $\delta(\bx_R)=0$. The bound on the right is distribution-free in the readout: it holds pointwise whatever the conditional-mean maps (affine or not) and whatever law generates the readouts, so every aggregation preserves it, and whenever $s_\reg<s_\star$ the fiber-conditional coverage remains below nominal at every readout and on every average. In the linear--Gaussian case $\delta(\bx_R)$ is affine in the readout and the aggregate over held-out truths has a closed form, reducing to \eqref{eq:coverage-gen} when the mean gap is constant. Reading the marginal spread in place of the conditional overstates the \emph{reported} blind spread by $\bm v^\top\Sigma_\reg\bm v-s_\reg^2\ge0$; this orders reported widths only, not the coverage of the blind-marginal protocol of \cref{sec:experiments} against the conditional law.
\end{proposition}

\subsection{The near-null regime (\cref{p:nearnull,p:mapnearnull})}
The exact statements so far use a genuine kernel; the seismic operator of \cref{sec:experiments} instead has a \emph{near-null}---i.e., unit directions $\bm v$ with $\|\bA\bm v\|$ small but nonzero---where both statements hold in quantitative form.
\begin{proposition}[The near-null freeze and undetectability, quantitatively \emph{(linear--Gaussian)}]\label{p:nearnull}\label{r:nearnull}
With $\bS_y$ as above, write $\bS_\reg=\bA\Sigma_\reg\bA^\top+\Gamma$ for the regularizer's data covariance. (i)~\emph{Freeze.} The one step moves the prior's precision by
\begin{equation}
\Sigma_{q_\reg}^{-1}-\Sigma_\reg^{-1}
=\bA^\top\big(\Gamma^{-1}-\Gamma^{-1}\bm Q^{-1}\Gamma^{-1}\big)\bA,
\qquad
\bm Q:=\Gamma^{-1}+\bS_y^{-1}-\bS_\reg^{-1}\succ0 .
\label{eq:nearnull-freeze}
\end{equation}
Sandwiching \eqref{eq:nearnull-freeze} with $\bN$ recovers the exact freeze \eqref{eq:freeze}, since $\bA\bN=\mathbf 0$; sandwiching with a unit vector $\bm v$ gives $\big|\bm v^\top(\Sigma_{q_\reg}^{-1}-\Sigma_\reg^{-1})\bm v\big|\le c\,\|\bA\bm v\|^2$ with $c=\|\Gamma^{-1}-\Gamma^{-1}\bm Q^{-1}\Gamma^{-1}\|_2\le\lambda_{\min}(\Gamma)^{-1}\max\!\big\{1,\lambda_{\max}(\bS_y)/\lambda_{\min}(\Gamma)\big\}$ independent of $\bm v$: a right-singular direction with singular value $\sigma_v$ has its precision moved by at most $c\,\sigma_v^2$. (ii)~\emph{Undetectability.} Two truth distributions that differ by a spread misreport of size $w>0$ along a unit vector $\bm v$, $\Sigma_\star\mapsto\Sigma_\star+w\,\bm v\bm v^\top$ (the near-null analogue of the witness \eqref{eq:witness}), induce data marginals separated per measurement by
\begin{equation}
\KL=\tfrac12\big(t-\log(1+t)\big)\le\tfrac{t^2}{4},
\qquad
t=w\,(\bA\bm v)^\top\bS_y^{-1}(\bA\bm v)\ \le\ \frac{w\,\|\bA\bm v\|^2}{\lambda_{\min}(\Gamma)} .
\label{eq:nearnull-kl}
\end{equation}
For $n$ independent observations, every test between the two truths has summed type-I and type-II errors at least $1-\sqrt{n\,\KL/2}$, which by \eqref{eq:nearnull-kl} exceeds $1/2$ whenever
\[
n\ <\ n_\star:=\frac{2}{t^2}\ \ge\ 2\left(\frac{\lambda_{\min}(\Gamma)}{w\,\|\bA\bm v\|^2}\right)^{\!2},
\]
so below the crossover $n_\star$ no test separates the two truths with summed errors below $1/2$. At $\bA\bm v=\bm0$ the two marginals coincide and $n_\star=\infty$, the exact non-identifiability of \cref{c:nonident}.
\end{proposition}
\noindent\emph{Proof.} \emph{(i)} By the law of total covariance (as in the proof of \cref{r:map-lg}) the curated covariance is $\Sigma_{q_\reg}=\Sigma_{\mathrm{post}}+\bG\bS_y\bG^\top$, and Woodbury gives
\[
\Sigma_{q_\reg}^{-1}=\Sigma_{\mathrm{post}}^{-1}-\Sigma_{\mathrm{post}}^{-1}\bG\big(\bS_y^{-1}+\bG^\top\Sigma_{\mathrm{post}}^{-1}\bG\big)^{-1}\bG^\top\Sigma_{\mathrm{post}}^{-1}.
\]
Directly from the definitions $\Sigma_{\mathrm{post}}^{-1}\bG=\bA^\top\Gamma^{-1}$, and
\[
\bG^\top\Sigma_{\mathrm{post}}^{-1}\bG=\Gamma^{-1}\bA\Sigma_{\mathrm{post}}\bA^\top\Gamma^{-1}=\Gamma^{-1}-\bS_\reg^{-1},
\]
since $\bA\Sigma_{\mathrm{post}}\bA^\top=\bm D-\bm D\bS_\reg^{-1}\bm D=\bm D\bS_\reg^{-1}\Gamma=\Gamma-\Gamma\bS_\reg^{-1}\Gamma$ for $\bm D=\bA\Sigma_\reg\bA^\top$ (Woodbury again, then $\bS_\reg=\bm D+\Gamma$ twice). Substituting, with $\Sigma_{\mathrm{post}}^{-1}=\Sigma_\reg^{-1}+\bA^\top\Gamma^{-1}\bA$,
\[
\Sigma_{q_\reg}^{-1}-\Sigma_\reg^{-1}
=\bA^\top\Gamma^{-1}\bA-\bA^\top\Gamma^{-1}\bm Q^{-1}\Gamma^{-1}\bA,
\]
which is \eqref{eq:nearnull-freeze}. Here $\bm Q\succ0$ because $\bS_\reg\succeq\Gamma$ makes $\Gamma^{-1}-\bS_\reg^{-1}\succeq0$ while $\bS_y^{-1}\succ0$. For the constant, $\Gamma^{-1}$ and $\Gamma^{-1}\bm Q^{-1}\Gamma^{-1}$ are each positive semidefinite, so their difference has spectral norm at most the larger of theirs; with $\bm Q\succeq\bS_y^{-1}$ giving $\Gamma^{-1}\bm Q^{-1}\Gamma^{-1}\preceq\Gamma^{-1}\bS_y\Gamma^{-1}$, so $c\le\max\{\|\Gamma^{-1}\|_2,\|\Gamma^{-1}\bS_y\Gamma^{-1}\|_2\}$, the stated bound.

\emph{(ii)} With $\bm u=\bA\bm v$ the perturbation affects the data only through $\bA$, so the perturbed data marginal is $\mathcal N(\bA\mu_\star,\bS_y+w\,\bm u\bm u^\top)$. The matrix $\bS_y^{-1}(\bS_y+w\bm u\bm u^\top)$ has eigenvalues $1$ (multiplicity $m-1$) and $1+t$, so the Gaussian divergence is $\tfrac12[\mathrm{tr}-m-\log\det]=\tfrac12(t-\log(1+t))$, and $\log(1+t)\ge t-t^2/2$ gives the quadratic bound in \eqref{eq:nearnull-kl}; since $\bS_y\succeq\Gamma$, $t\le w\|\bm u\|^2/\lambda_{\min}(\bS_y)\le w\|\bm u\|^2/\lambda_{\min}(\Gamma)$.
Over $n$ independent observations the divergences add, so Pinsker's inequality and Le Cam's two-point bound give every test summed type-I and type-II error at least $1-\sqrt{n\,\KL/2}$, exceeding $1/2$ exactly when $n\,\KL/2<1/4$, guaranteed by $n<2/t^2$. A misreported blind mean, $\mu_\star\mapsto\mu_\star+w\bm v$, gives, with its own crossover $n_\star=1/\beta$ at $\beta=w^2(\bA\bm v)^\top\bS_y^{-1}(\bA\bm v)$ (the divergence being linear rather than quadratic in the misreport, so the constant is not the spread case's),
\[
\KL=\tfrac12\,w^2(\bA\bm v)^\top\bS_y^{-1}(\bA\bm v)
\]
and the same conclusion. \qed

The collapse degrades gracefully on the near-null for penalties strongly convex along the direction, and not under convexity alone. Since the damping that mitigates the classical ill-posed noise amplification is the same constant that enforces the pinning, reading \cref{p:cover} in it makes the tradeoff explicit: a Gaussian regularizer with curvature $\mu_v$ along a blind direction reports a conditional spread of at least $\mu_v^{-1/2}$ there, so the coverage its central-$(1-\alpha)$ band attains against a truth of spread $s_\star$ and matching blind mean (a mean gap only lowers it, \cref{p:cover}) satisfies
\begin{equation}
C \;\ge\; 2\,\Phi\!\Big(\frac{z}{s_\star\sqrt{\mu_v}}\Big)-1 ,
\label{eq:covmu}
\end{equation}
with equality when the blind subspace is one-dimensional; on a larger one the curvature alone underestimates the fiber spread away from the blind block's eigendirections, so the display bounds the shortfall rather than giving it. The right side is a decreasing function of the damping, and so, by the Loewner monotonicity of the fiber conditional in the damping, is the coverage itself: stability and calibration on the blind subspace trade off along a single axis. Separately, an improper-penalty archive collapses only at the exact kernel~(\cref{p:map}), so the migration archive of \cref{sec:experiments}, undamped along its near-null, has its blind tightness there as an empirical reading rather than this proposition's guarantee.
\begin{proposition}[The near-null single-best collapse, and its sharp boundary \emph{(linear operator; Gaussian data term where stated)}]\label{p:mapnearnull}\label{r:mapnearnull}
Fix a unit $\bm v$, let $\bm\Pi=\bm I-\bm v\bm v^\top$ collect the remaining coordinates, and let the archive store \eqref{eq:varest} with linear $\bF=\bA$, the data term differentiable at $\bA\hat\bx(\by)$ with gradient $\bm d(\by)$ there, and the penalty convex, finite, with $\mu_v$-strongly convex sections along $\bm v$: $t\mapsto R(\bx+t\bm v)-\tfrac{\mu_v}{2}t^2$ convex for every $\bx$, $\mu_v>0$ (a quadratic penalty of precision $\bm P$ has $\mu_v=\bm v^\top\bm P\bm v$; any convex penalty plus damping $\tfrac{\lambda}{2}\|\bx\|^2$ has $\mu_v\ge\lambda$). Write $\psi(\bm\Pi\bx)$ for the penalty's unique conditional minimizer along $\bm v$ given the remaining coordinates $\bm\Pi\bx$. Then, for every measurement and every minimizer:
(i)~$|\bm v^\top\hat\bx-\psi(\bm\Pi\hat\bx)|\le|\langle\bA\bm v,\bm d(\by)\rangle|/\mu_v\le\|\bA\bm v\|\,\|\bm d(\by)\|/\mu_v$, and for a quadratic penalty the first bound is an identity;
(ii)~consequently $\E[\operatorname{Var}(\bm v^\top\hat\bx\mid\bm\Pi\hat\bx)]\le\|\bA\bm v\|^2\,\E\|\bm d\|^2/\mu_v^2$ (with $\E\|\bm d\|^2$ bounded uniformly as $\|\bA\bm v\|\to0$ for the Gaussian data term, a penalty bounded below, and a measurement law of finite second moment), and against a truth whose fiber conditional given the readout has density at most $B_\star$, the narrowest conditional $1-\alpha$ credible interval covers with probability at most $2B_\star\|\bA\bm v\|\,\E\|\bm d\|/(\alpha\mu_v)$: the collapse of \cref{p:map}, recovered at rate $\|\bA\bm v\|$ as the illumination vanishes;
(iii)~for a Gaussian data term and a ridge penalty the blind \emph{marginal} collapses as well, $\operatorname{Var}(\bm v^\top\hat\bx)\le\lambda_{\max}(\tilde{\bm S}_y)\,\|\tilde\bA\bm v\|^2/\lambda^2$ in whitened coordinates ($\tilde\bA=\Gamma^{-1/2}\bA$, $\tilde{\bm S}_y=\Gamma^{-1/2}\bS_y\Gamma^{-1/2}$, local to this proposition), while for a coupled quadratic penalty the marginal stays of order one (the graph-transported spread of $\psi$), and only the conditional statement holds: the near-null form of the alignment dichotomy in \cref{r:map-lg}.
\end{proposition}
\noindent\emph{Proof.} \emph{(i)} At a minimizer, Fermat's rule for the sum of a function differentiable at the point and a finite convex function supplies $\bm s=-\bA^\top\bm d(\by)\in\partial R(\hat\bx)$. On the line along $\bm v$, $f(t)=R(\bm\Pi\hat\bx+t\bm v)$ is $\mu_v$-strongly convex with $\bm v^\top\bm s\in\partial f(\bm v^\top\hat\bx)$ and $0\in\partial f(\psi(\bm\Pi\hat\bx))$, so subdifferential monotonicity gives $\mu_v\,|\bm v^\top\hat\bx-\psi|^2\le(\bm v^\top\bm s)(\bm v^\top\hat\bx-\psi)$, the bound in (i); for a quadratic penalty, stationarity solved along $\bm v$ makes it the identity $\bm v^\top\hat\bx-\psi(\bm\Pi\hat\bx)=-\langle\bA\bm v,\bm d\rangle/(\bm v^\top\bm P\bm v)$.

\emph{(ii)} Strong convexity of the sections makes $\psi$ single-valued and continuous, so $\psi(\bm\Pi\hat\bx)$ is one measurable predictor of $\bm v^\top\hat\bx$; the conditional mean, being $L^2$-optimal, does no worse, so squaring (i) and taking expectations gives the variance bound of (ii). Boundedness of $\E\|\bm d\|^2$ follows from comparing the objective at $\hat\bx$ with a fixed reference $\bx_0$,
\[
\|\by-\bA\hat\bx\|^2_{\Gamma^{-1}}\le\|\by-\bA\bx_0\|^2_{\Gamma^{-1}}+2\big(R(\bx_0)-\inf R\big),
\]
the Gaussian gradient being controlled by the misfit. For coverage, (i) and the conditional Markov inequality place conditional mass at least $1-\alpha$ on $\psi(\bm\Pi\hat\bx)\pm\|\bA\bm v\|\,\E[\|\bm d\|\mid\bm\Pi\hat\bx]/(\alpha\mu_v)$, so the narrowest $1-\alpha$ credible interval is no wider; a $B_\star$-bounded conditional density gives it conditional probability at most $2B_\star\|\bA\bm v\|\,\E[\|\bm d\|\mid\bm\Pi\hat\bx]/(\alpha\mu_v)$, and the tower property averages this to the stated bound.

\emph{(iii)} Whitening and the push-through identity $(\tilde\bA^\top\tilde\bA+\lambda\bm I)^{-1}\tilde\bA^\top=\tilde\bA^\top(\tilde\bA\tilde\bA^\top+\lambda\bm I)^{-1}$ express $\bm v^\top\hat\bx$ as a linear image of the whitened data with the stated variance bound; a two-dimensional coupled-precision counterexample with $\bA\bm v=\bm 0$ and off-diagonal precision has order-one blind marginal, ruling out any marginal bound off alignment. \qed

\emph{Sharpness.} The strong convexity cannot be weakened to convexity, and both failure modes are exactly solvable. \emph{A kinked penalty} (one dimension, $\bA=\varepsilon>0$, Gaussian data term, $R(x)=|x-1|+|x+1|$) has minimizer $\hat x(\by)=\by/\varepsilon$ for $|\by|<\varepsilon$ and $\hat x(\by)=\operatorname{sgn}(\by)$ on a widening interval beyond, so its spread stays of order one as $\varepsilon\to0$ while at $\varepsilon=0$ the fixed-selection collapse of \cref{p:map} reports zero width: the near-null spread is discontinuous at the kernel, and no bound $C\|\bA\bm v\|^\kappa$, for any $C$ and $\kappa>0$, can hold for merely convex penalties; the data tilt, however weak, moves the minimizer between kink locations a fixed distance apart. \emph{A penalty exactly flat along $\bm v$} (e.g., total variation along its constant direction) leaves the coordinate to the data term alone, and a separable two-dimensional witness has $\operatorname{Var}(\bm v^\top\hat\bx)\sim\|\bA\bm v\|^{-2}$: the classical ill-posed noise amplification, opposite in kind to the freeze. The two regimes are separated exactly by $\mu_v$.
\subsection{The smoothing floor (\cref{r:map-lg})}

\begin{remark}[The linear--Gaussian sharpening and the estimator's smoothing floor]\label{r:map-lg}
In the linear--Gaussian case with $\Sigma_\reg\succ0$---writing $q^{\mathrm{ps}}_\reg\equiv q_\reg$ for the posterior-sample curated prior and $q^{\mathrm{sb}}_\reg$ for its single-best counterpart---single-best curation discards exactly the data-averaged posterior covariance,
\begin{equation}
\Sigma_{q^{\mathrm{sb}}_\reg}=\Sigma_{q^{\mathrm{ps}}_\reg}-\E_\by[\mathrm{Cov}(\bx\mid\by)],
\label{eq:totalcov}
\end{equation}
which on the blind fiber is the entire reported spread. Under alignment the blind \emph{marginal} then collapses to a point mass; off alignment it collapses to a deterministic image of the resolved coordinates, with nonzero marginal spread but zero conditional width. A generative model cannot represent a point mass, so fit to the MAP archive it reports a blind spread set not by the data but by its own smoothing scale $h$ (e.g., kernel width, diffusion residual-noise floor, or flow regularization).
\end{remark}
\emph{Gaussian sharpening.} Here the posterior $\mathcal N(\mu_{\mathrm{post}}(\by),\Sigma_{\mathrm{post}})$ has its mode at its mean, so
\begin{equation}
\mathrm{MAP}_\reg(\by)=\mu_{\mathrm{post}}(\by)=\bG\by+(\bm I-\bG\bA)\mu_\reg,\qquad \bG=\Sigma_{\mathrm{post}}\bA^\top\Gamma^{-1},
\end{equation}
affine in $\by$. Its pushforward of $\truth(\by)=\mathcal N(\bA\mu_\star,\bS_y)$ has covariance $\Sigma_{q^{\mathrm{sb}}_\reg}=\bG\bS_y\bG^\top$, while the law of total covariance for the posterior-sample archive ($\bx\sim\pi_\reg(\cdot\mid\by),\ \by\sim\truth$) gives
\begin{equation}
\Sigma_{q^{\mathrm{ps}}_\reg}=\underbrace{\E_\by[\mathrm{Cov}(\bx\mid\by)]}_{=\ \Sigma_{\mathrm{post}}}+\underbrace{\mathrm{Cov}_\by[\E(\bx\mid\by)]}_{=\ \bG\bS_y\bG^\top}
\ \Longrightarrow\ \Sigma_{q^{\mathrm{sb}}_\reg}=\Sigma_{q^{\mathrm{ps}}_\reg}-\Sigma_{\mathrm{post}},
\end{equation}
which is \eqref{eq:totalcov}. As $\mathrm{rank}\,\bG=\mathrm{rank}\,\bA$ and $\bS_y\succ0$, $\Sigma_{q^{\mathrm{sb}}_\reg}$ has rank exactly $\mathrm{rank}\,\bA$, with range $\bG(\mathbb R^m)=\Sigma_{\mathrm{post}}\,\mathrm{row}(\bA)$. This range meets $\ker\bA$ only at $\bm0$: if $\bm0\ne\Sigma_{\mathrm{post}}\bm w\in\ker\bA$ with $\bm w\in\mathrm{row}\,\bA$, then $\bM\Sigma_{\mathrm{post}}\bm w=\bm0$ and $\Sigma_{\mathrm{post}}^{-1}=\Sigma_\reg^{-1}+\bM$ give $\Sigma_\reg^{-1}\Sigma_{\mathrm{post}}\bm w=\bm w$, i.e.\ $\Sigma_{\mathrm{post}}\bm w=\Sigma_\reg\bm w$: a blind vector $\Sigma_\reg\bm w\in\ker\bA$ with $\bm w\in\mathrm{row}\,\bA=(\ker\bA)^\perp$, whence $\bm w^\top\Sigma_\reg\bm w=0$ and $\bm w=\bm0$, a contradiction. Thus the blind-conditional variance is zero and any nonzero blind \emph{marginal} variance off alignment is a deterministic image of the resolved coordinates.

\emph{Smoothing floor.} A generative estimator that fits the point mass $\delta_{\hat\bx_B}$ with smoothing scale $h$ convolves it with a Gaussian kernel of width $h\,c_K$, giving blind conditional $\mathcal N(\hat\bx_B,(h c_K)^2)$ and, at a centered atom, coverage
\begin{equation}
C^{\mathrm{sb}}=2\Phi\!\big(z\,h\,c_K/s_\star\big)-1\ \xrightarrow[h\to0]{}\ 0 .
\end{equation}
This holds exactly for a Gaussian-kernel density estimator ($c_K$ the bandwidth constant) and for an exact-score diffusion stopped at noise level $h$.

More generally, let the output law be the archive's law smoothed by an independent mollifier of scale $h$ with finite second moment (blind-component standard deviation $hc_B$, resolved second moment $(hm_R)^2$), and let the fiber selection be $L$-Lipschitz on the readouts. The conditional mean is the $L^2$-optimal predictor and the selection evaluated at the smoothed readout is one candidate, so Minkowski gives $\E_{\bx_R}[\operatorname{Var}(\bm v^\top\bx_B\mid r)]\le h^2(c_B+Lm_R)^2$: the reported blind conditional spread, and by Chebyshev its coverage of a truth with a bounded fiber density, vanish with the smoothing scale under any such mollifier; the exact Gaussian formula above remains the sharp constant, and the passage to a learned, discretized sampler remains empirical.

\subsection{The safe regime (\cref{p:recover})}
\begin{proposition}[Repeating the loop against fresh data recovers only the resolved directions \emph{(freeze fully general; convergence linear--Gaussian)}]\label{p:recover}
Iterate $\pi_{t+1}=\loopop[\pi_t]$ from $\pi_0=\reg$, each round re-inverting fresh real measurements into posterior samples.
(i)~\emph{Freeze.} For any likelihood that depends on $\bx$ only through $\bA\bx$, any prior, and any per-round measurement-averaging law, the blind-fiber \emph{conditional} is frozen for all rounds (the freeze of \cref{p:blind}, iterated): in the $\mathrm{row}(\bA)\oplus\ker\bA$ split, $\pi_t(\bx_B\mid\bx_R)=\reg(\bx_B\mid\bx_R)$ for every $t$.
(ii)~\emph{Convergence.} In the linear--Gaussian case, suppose $\bA$ has full row rank, $\Sigma_\star\succ0$, and the resolved covariance is started nondegenerate, $\bA\Sigma_\reg\bA^\top\succ0$. Then the loop preserves Gaussianity, $\pi_t=\mathcal N(\mu_t,\Sigma_t)$, and the resolved marginal converges to the truth's ($\bP_R\Sigma_t\bP_R\to\bP_R\Sigma_\star\bP_R$ and $\bP_R\mu_t\to\bP_R\mu_\star$, the mean error decaying geometrically), so $\pi_t\to\truth$ on $\mathrm{row}(\bA)$ while every blind-fiber conditional stays at the regularizer's.
Consequently, under the hypotheses of~(ii), recovery of the full truth is guaranteed for \emph{every} regularizer if and only if $\ker\bA=\{\bm0\}$.
\end{proposition}
\noindent\emph{Proof.} \emph{(i)} The cancellation in \eqref{eq:freeze-cond} uses only that the reweight $g_t(\bx)=\E_{\by\sim\truth}\!\big[p(\by\mid\bx)/\pi_t(\by)\big]$ is constant along blind fibers, which holds whenever the likelihood depends on $\bx$ only through $\bA\bx$, for any prior in the role of $\reg$ and any averaging law in place of $\truth(\by)$; induction from $\pi_0=\reg$ gives the claim. Only the fiber \emph{conditional} is frozen: off alignment the blind \emph{marginal} moves with the evolving resolved marginal through the frozen regression~(\cref{r:map-lg}).

\emph{(ii)} The loop maps Gaussians to Gaussians,
\[
\Sigma_{t+1}=\Sigma_{\mathrm{post},t}+\bG_t\bS_y\bG_t^\top,\qquad
\mu_{t+1}=\mu_t+\bG_t\big(\bA\mu_\star-\bA\mu_t\big),
\]
with $\bG_t=\Sigma_t\bA^\top(\bm D_t+\Gamma)^{-1}$ and $\Sigma_{\mathrm{post},t}=\Sigma_t-\bG_t\bA\Sigma_t$, a form requiring only $\bm D_t+\Gamma\succ0$, so a singular $\Sigma_\reg$ is admissible. Via the Woodbury rearrangement $\bA\Sigma_{\mathrm{post},t}\bA^\top=\bm D_t(\bm D_t+\Gamma)^{-1}\Gamma$, the covariance step projects to an \emph{autonomous} recursion in $\bm D_t:=\bA\Sigma_t\bA^\top$,
\[
\begin{gathered}
\bm D_{t+1}=\Xi(\bm D_t):=\bm K_t\Gamma+\bm K_t\bS_y\bm K_t^\top,\qquad \bm K_t=\bm D_t(\bm D_t+\Gamma)^{-1},\\
\bS_y=\bm D_\star+\Gamma,\ \ \bm D_\star=\bA\Sigma_\star\bA^\top ,
\end{gathered}
\]
and, since $\bA\bG_t=\bm K_t$, the mean step to $\bA\mu_{t+1}-\bA\mu_\star=(\bm I-\bm K_t)(\bA\mu_t-\bA\mu_\star)$; the hypotheses give $\bm D_\star\succ0$ and $\bm D_0\succ0$. Whiten $\widetilde{\bm D}_t=\Gamma^{-1/2}\bm D_t\Gamma^{-1/2}$ and $\widetilde{\bS}=\bm I+\Gamma^{-1/2}\bm D_\star\Gamma^{-1/2}$, so that
\[
\widetilde{\bm D}_{t+1}=\widetilde{\bm K}_t+\widetilde{\bm K}_t\widetilde{\bS}\,\widetilde{\bm K}_t,\qquad
\widetilde{\bm K}_t=\widetilde{\bm D}_t\big(\widetilde{\bm D}_t+\bm I\big)^{-1}\ \text{symmetric},
\]
with $\|\widetilde{\bm D}_{t+1}\|\le1+\|\widetilde{\bS}\|$. Write $\lambda_t=\lambda_{\min}(\widetilde{\bm D}_t)$ and $s_{\min}=\lambda_{\min}(\widetilde{\bS})>1$. Since $\widetilde{\bm K}_t=f(\widetilde{\bm D}_t)$ with $f(\lambda)=\lambda/(1+\lambda)$ increasing, $\lambda_{\min}(\widetilde{\bm K}_t)=f(\lambda_t)$; congruence gives $\widetilde{\bm K}_t\widetilde{\bS}\widetilde{\bm K}_t\succeq s_{\min}\widetilde{\bm K}_t^2$; and Weyl's inequality for the sum yields
\[
\lambda_{t+1}\;\ge\;g_\flat(\lambda_t),\qquad
g_\flat(\lambda)=\frac{\lambda}{1+\lambda}+s_{\min}\,\frac{\lambda^2}{(1+\lambda)^2},\qquad
g_\flat(\lambda)-\lambda=\frac{\lambda^2\,(s_{\min}-1-\lambda)}{(1+\lambda)^2},
\]
so $g_\flat$ is increasing, fixes $s_{\min}-1$, and satisfies $g_\flat(\lambda)\ge\lambda$ for $\lambda\le s_{\min}-1$; by induction
\[
\lambda_{\min}(\widetilde{\bm D}_t)\;\ge\;\lambda_\flat:=\min\!\big(\lambda_{\min}(\widetilde{\bm D}_0),\,s_{\min}-1\big)\;>\;0\qquad\text{for all }t,
\]
a uniform interior floor, needed because the rank-deficient boundary faces are themselves fixed points of $\Xi$, which ascent alone cannot exclude. If $\widetilde{\bm D}\succ0$ is fixed, $\bm E=\widetilde{\bm D}+\bm I$ turns the fixed-point equation into
\[
(\bm E-\bm I)\bm E^{-1}(\bm E-\bm I)=(\bm E-\bm I)\bm E^{-1}\widetilde{\bS}\,\bm E^{-1}(\bm E-\bm I),
\]
and cancelling the invertible factors $(\bm E-\bm I)\bm E^{-1}$ and $\bm E-\bm I$ leaves $\bm E=\widetilde{\bS}$: the interior fixed point is unique, $\bm D=\bm D_\star$.

Finally, $\Xi$ is the Gaussian EM/NPMLE update for the data-space model $\by=\bm\upsilon+\varepsilon$, $\bm\upsilon\sim\mathcal N(0,\bm D)$, $\varepsilon\sim\mathcal N(0,\Gamma)$, fit to the truth's marginal, with objective $J(\bm D)=\KL\big(\mathcal N(0,\bS_y)\,\big\|\,\mathcal N(0,\bm D+\Gamma)\big)$ nonincreasing along the trajectory and strictly decreasing unless $\Xi(\bm D)=\bm D$~\citep{dempster1977em}. The floor keeps the bounded iterates in a compact subset of the interior, on which $\Xi$ and $J$ are continuous, so the standard descent argument makes every limit point an interior fixed point; there is exactly one, hence $\bm D_t\to\bm D_\star$. For the means, $\bm I-\widetilde{\bm K}_t=(\widetilde{\bm D}_t+\bm I)^{-1}$ has norm at most $(1+\lambda_\flat)^{-1}<1$, so $\bA\mu_t\to\bA\mu_\star$ geometrically. Both conclusions lift through the pseudoinverse ($\bA^{+}\bA=\bP_R$): $\bP_R\Sigma_t\bP_R=\bA^{+}\bm D_t(\bA^{+})^\top\to\bP_R\Sigma_\star\bP_R$ and $\bP_R\mu_t\to\bP_R\mu_\star$.

If $\ker\bA=\{\bm0\}$ the resolved marginal is the whole prior and (ii) gives $\pi_t\to\truth$ for every regularizer; if not, any regularizer whose blind-fiber conditional differs from the truth's stays wrong there for every round, by (i); the one whose blind conditional coincides is exactly the agreement no measurement can verify~(\cref{c:nonident}). Beyond the linear--Gaussian case only the freeze and the ascent of \cref{p:law} remain: the convergence argument relies on the loop's Gaussian closure, and the loop is alternating minimization of an $I$-divergence between two convex sets, whose objective decreases by the classical argument~\citep{csiszar1984}; what that argument does not supply, and what the proof above provides, is convergence of the iterates through the rank-deficient boundary faces. \qed

\section{Operators, priors, and training}
\label{app:experimental}

\subsection{The two-dimensional example}\label{app:toy}
The unknown $\bx=(x_R,x_B)\in\R^2$ is observed through $\bA=[1,0]$ as $y=x_R+\varepsilon$ with $\varepsilon\sim\mathcal N(0,s^2)$ and $s=0.25$, so the second coordinate spans the blind subspace. The regularizer is the correlated Gaussian $\reg=\mathcal N\bigl(\bm 0,\bigl(\begin{smallmatrix}1&0.6\\0.6&1\end{smallmatrix}\bigr)\bigr)$, whose correlation allows the blind marginal to move under curation while every blind conditional remains frozen. The truth is an equal-weight mixture of two Gaussians with means $(-1.2,2)$ and $(1.2,-2)$ and the common covariance $\bigl(\begin{smallmatrix}0.35&1\\1&9\end{smallmatrix}\bigr)$, so its blind conditionals are wider than the regularizer's and bimodal on some fibers. Every archive is drawn through one pipeline, a measurement $y\sim\truth(y)$ followed by one posterior sample for a posterior-sample archive or by the global posterior mode for a single-best archive. Under the regularizer both archive laws are closed form, since the curated prior of \eqref{eq:curated} is an exact two-component mixture and the single-best archive lies on the regularizer's regression line.

The priors of \cref{fig:toy-money} are fit to $4000$ samples of each of the three training sets, the regularizer's two archives and the truths themselves, with $1000$ more held out for the flow's early stopping. Every set is dequantized in the same absolute units by adding $\mathcal N(\bm 0,h^2\bm I)$ with $h=0.25$, so each fit approximates the dequantized law, whose components gain $h^2\bm I$ and whose blind conditionals therefore remain exact Gaussian mixtures. Two model classes are fit, a two-component Gaussian mixture by expectation--maximization and a HINT flow of the family used on the groundwater operator. The flow has eight coupling layers of tree depth one and three-layer networks of width $64$, and is trained by Adam at learning rate $2\times10^{-3}$ with three seeds.

The reported blind spread is the standard deviation of each fitted density's conditional on the fiber $x_R=0$, and the dashed lines of \cref{fig:toy-money} are the same quantity for the exact dequantized law. On the single-best archive that law is the dequantization itself, widened through the slope of the regularizer's regression line, which is the smoothing floor of \cref{r:map-lg} in closed form.

\subsection{The seismic operator}\label{app:seismicop} The first physical operator linearizes the acoustic wave equation. For pressure $u(\bm r,t)$ at position $\bm r\in\Omega\subset\mathbb R^2$, $t\in[0,T]$, squared slowness $m(\bm r)=1/c(\bm r)^2$, and a source at $\bm r_s$,
\begin{equation}
m(\bm r)\,\partial_t^2 u-\Delta u=w(t)\,\delta(\bm r-\bm r_s),\qquad u(\bm r,0)=\partial_t u(\bm r,0)=0,
\label{eq:wave}
\end{equation}
with $w$ a Ricker wavelet of peak frequency $f_0=30$~Hz and absorbing (damping-layer) conditions on $\partial\Omega$; the data are the traces $u(\bm r_r,t)$ at receivers $\bm r_r$. Writing $m=m_0+\delta m$ about a smooth background $m_0$ and $u=u_0+\delta u$, the Born (single-scattering) approximation gives
\begin{equation}
m_0\,\partial_t^2 u_0-\Delta u_0=w(t)\,\delta(\bm r-\bm r_s),\qquad m_0\,\partial_t^2\delta u-\Delta\delta u=-\,\delta m\,\partial_t^2 u_0,
\label{eq:born}
\end{equation}
under the same conditions. The forward (demigration) operator maps reflectivity to shot records, $(\bA\,\delta m)_{s,r}(t)=\delta u(\bm r_r,t)$; its adjoint is reverse-time migration, and the legacy reconstruction is least-squares migration~\citep{nemeth1999leastsquares} under an $\ell_1$ penalty, $\tfrac12\|\bA\,\delta m-\by\|_2^2+\lambda\|\delta m\|_1$ approximately minimized by a fixed number of stochastic proximal iterations.

In practice, we solve \eqref{eq:wave}--\eqref{eq:born} by $16$th-order finite differences~\citep{devito-api,devito-compiler} with a $40$-cell damping layer over a depth-gradient background beneath a water layer ($m_0$, $256\times256$ grid at $20\times12.5$~m spacing, record length $T=2$~s), with a shallow line of $12$ sources and $24$ receivers, and reflectivity from the Kirchhoff-migrated Parihaka-3D field survey~\citep{Veritas2005,WesternGeco2012,parihaka_segwiki}. Due to finite aperture and source band, the illumination $\|\bA\bm v\|^2$ decays sharply, e.g., for deep, steeply dipping, or sub-wavelength reflectivity~\citep{symes2009seismic}, so here the blind subspace is the operator's \emph{low-illumination} end---i.e., directions with $\|\bA\bm v\|^2$ below $1\%$ of the resolved median (\cref{fig:seismic-cutoff})---and \cref{r:nearnull,r:mapnearnull} quantify how the exact statements extend to it. The oracle prior is trained on broadband reflectivity built above the archive's band from the survey's own tracked horizons~\citep{WuFomel2018}, the curated prior on the least-squares-migration archive.

This is a controlled synthetic Born experiment rather than a field acquisition, and its geometry is deliberately sparse. The $2$~s record also limits the depth from which a reflection returns in time, so part of the deep blind end is blind because the recording ends.

\subsection{The groundwater operator}\label{ss:darcy} The second physical operator is the classical elliptic groundwater problem in the setting of \citet{beskos2017geometric}, which we use to extend the same experiment to a learned prior of a different family, a HINT normalizing flow. On $\Omega=[0,1]^2$ the log-permeability $u$ determines the hydraulic head $H$,
\begin{equation}
-\nabla\cdot\!\big(e^{u}\nabla H\big)=0,\quad H|_{x_2=0}=x_1,\ \ H|_{x_2=1}=1-x_1,\ \ \partial_{x_1}H|_{x_1\in\{0,1\}}=0,
\label{eq:darcy}
\end{equation}
solved on a $40\times40$ mesh, with the head read at $33$ sensors ($32$ on a circle, one at the center) under Gaussian noise. The unknown is the whitened Karhunen--Lo\`eve coefficient vector $\bm\xi$ of a Gaussian log-permeability prior (a Laplacian-eigenbasis covariance $(-\Delta)^{-1.1}$, truncated at $10\times10$ modes). The map $\bm\xi\mapsto\by$ is nonlinear, so its split is the local one fixed by the Jacobian, which we linearize once at the prior mean and keep fixed; away from that base state the null rotates with the field, which is why this example's blind reading is empirical~(\cref{r:nonlinear}).

This blind subspace consists of two components. With $33$ sensors against $d$ coefficients the Jacobian has an exact kernel of at least $d-33$ directions, and to it we add the tail whose singular values fall below the noise level, which the sensors see but cannot distinguish from noise. The legacy reconstruction is the single-best MAP estimate under the correct Gaussian prior (L-BFGS on whitened coefficients), the field standard, which returns the prior mean on the blind subspace to first order; the oracle prior is trained on the truths' Karhunen--Lo\`eve coefficients, the curated on the MAP archive.

\subsection{What is scored, and how}\label{app:metrics}
For a functional $\theta=\bm v^\top\bx$ ($\bm v$ resolved or blind) and a set of prior or posterior samples $\{\bx^{(i)}\}$, the central-$90\%$ interval is the empirical $[0.05,0.95]$ quantile band of $\{\bm v^\top\bx^{(i)}\}$, and its coverage is the fraction of held-out truths whose $\theta_\star$ falls within it.

Because the data fix the resolved scale, a single scalar $\kappa$ places each seismic prior on a common footing, matching its simulated-survey energy, noise-subtracted, to the observed surveys'. This source-calibration step uses only the known operator and noise of assumption~(ii) in \cref{sec:setup}, with no ground truth. Since $\kappa$ rescales a prior's resolved and blind spreads together, a mis-set $\kappa$ would appear on the resolved subspace, where the data are informative, before it could produce a blind shortfall.

The reported coverage is a blind-\emph{marginal} quantity, whereas \cref{p:cover,p:cover-na} score the fiber-\emph{conditional} band. Since the two coincide without resolved--blind coupling, while in general a marginal read can overstate or mask the conditional shortfall~(\cref{p:cover-na}), the reported coverages are protocol-matched comparisons between priors, not measurements of the closed-form law. The seismic read is on the $16$ certified blind and $32$ resolved dimensions the operator's illumination-ranked probe atoms span, a few of which have negligible truth variation and are retained, since removing them increases the reported shortfall.

Blind- and resolved-subspace coverage is computed from $384$ prior samples per trained prior (DDIM, $50$ steps) against $160$ held-out truths on the seismic operator, and over $24$ held-out truths per training seed on the groundwater operator. The groundwater posterior is drawn by preconditioned Crank--Nicolson (pCN) chains of $4.4\times10^4$ steps at $\beta=0.08$, with $10^3$ discarded as burn-in and the remainder thinned by five. Our code reproduces every record. \Cref{tab:coverage} summarizes the achieved coverages at the shared levels.

\subsection{Reconstruction, training, and samplers}

Within each oracle/curated pair the priors are matched in architecture, data budget, optimizer, and random-number state, reset before each run, so they differ only in the training target and in where early stopping halts each fit.

\paragraph{Seismic: acquisition and legacy reconstruction.} On the linearized-Born operator, wavelet-colored noise is added to every shot. The legacy archive is a sparsity-promoting least-squares migration: starting from zero, we minimize the shot-data misfit under an $\ell_1$ penalty on the reflectivity through five epochs of stochastic per-shot proximal gradient (one shot per step, AdaGrad-preconditioned to mitigate the illumination imbalance), with the soft threshold fixed from the unregularized migration's amplitude and the accumulated step size (value in the code).

\paragraph{Seismic: diffusion priors.} Each of the two priors is a denoising diffusion model~\citep{ho2020denoising} with a U-Net ($64$ base channels, attention at one scale) from the diffusers library~\citep{vonplaten2022diffusers}, with a cosine-$\beta$ schedule over $1000$ steps~\citep{nichol2021improved}, trained by AdamW under a one-cycle schedule (peak learning rate $2\times10^{-4}$, batch size $8$, $3000$ steps, gradient-norm clipping) on per-pixel-standardized images and sampled with the deterministic DDIM update~\citep{song2021ddim} used for every scored ensemble, the sampler's default $[-1,1]$ output clamp disabled, since it collapses standardized samples.

As expected, both priors fit their targets: their unconditional samples resemble the training images~(\cref{fig:seismic-training}), and the validation denoising loss falls to a plateau~(\cref{fig:seismic-loss}).

\begin{figure}[!b]
\centering
\includegraphics[width=\linewidth]{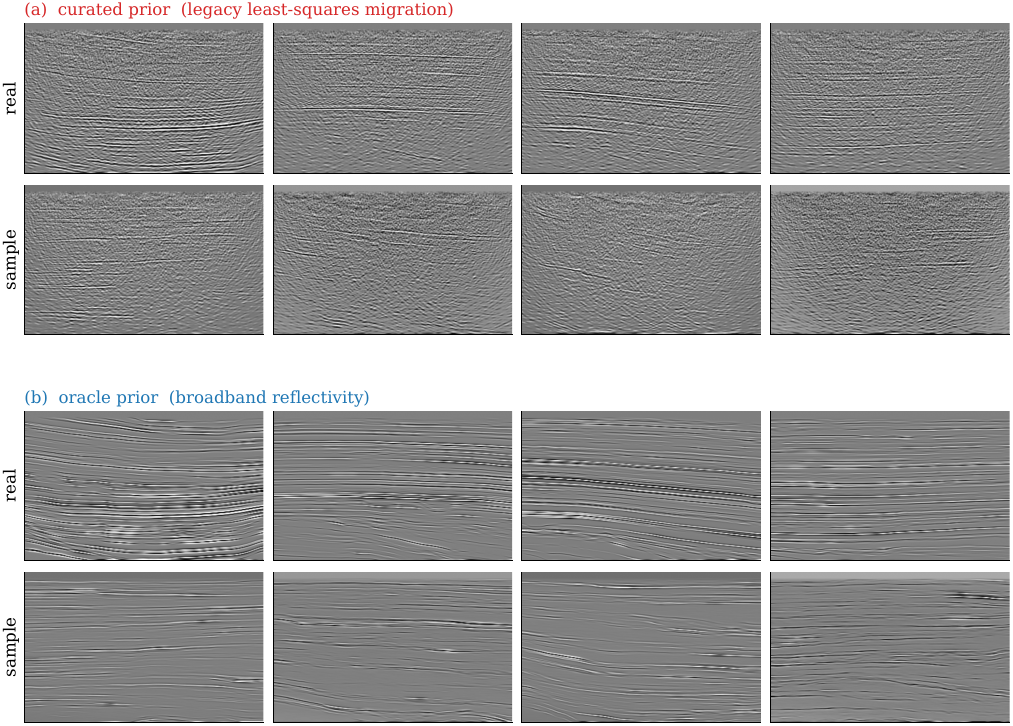}
\caption{Real training images alongside new unconditional DDPM samples for the (a)~\textcolor{madcol}{curated} (legacy least-squares migrations) and (b)~\textcolor{emcol}{oracle} (broadband reflectivity) seismic priors, on one shared grayscale.}
\label{fig:seismic-training}
\end{figure}

\begin{figure}[!htbp]
\centering
\includegraphics[width=\linewidth]{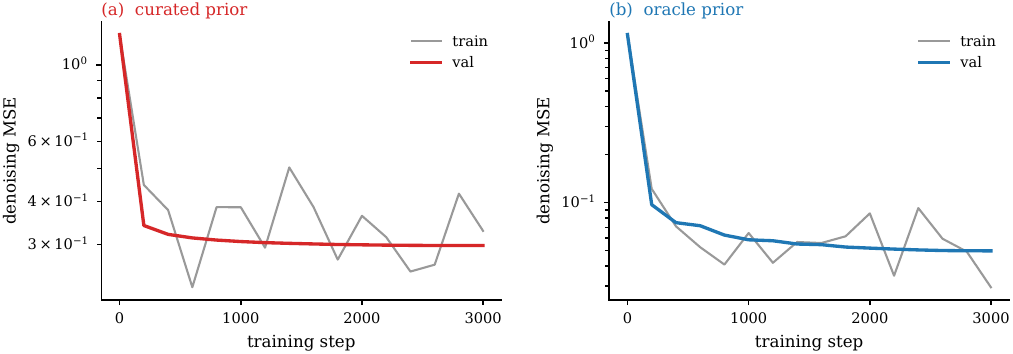}
\caption{Train and validation denoising-MSE (log scale) for the (a)~\textcolor{madcol}{curated} and (b)~\textcolor{emcol}{oracle} seismic diffusion priors.}
\label{fig:seismic-loss}
\end{figure}

\paragraph{Groundwater: flow priors.} Each of the two priors is a HINT normalizing flow~\citep{kruse2021hint}, a binary-tree affine-coupling flow (four coupling blocks, tree depth two, three-layer coupling networks) over the $d=100$ Karhunen--Lo\`eve coefficients, trained by maximum likelihood with per-feature standardization, Adam under a polynomial-decay schedule, a small dequantization noise that removes the low-dimensional over-concentration the collapsed archive would otherwise induce in the flow, and validation-based early stopping.

Since the oracle flow starts close to its Gaussian base and its target is itself Gaussian, its likelihood remains near the entropy floor from the start and early stopping halts it there; the curated target instead exhibits the archive's low-dimensional structure, which the flow does fit~(\cref{fig:darcy-training,fig:darcy-loss}).

\begin{figure}[!htbp]
\centering
\includegraphics[width=\linewidth]{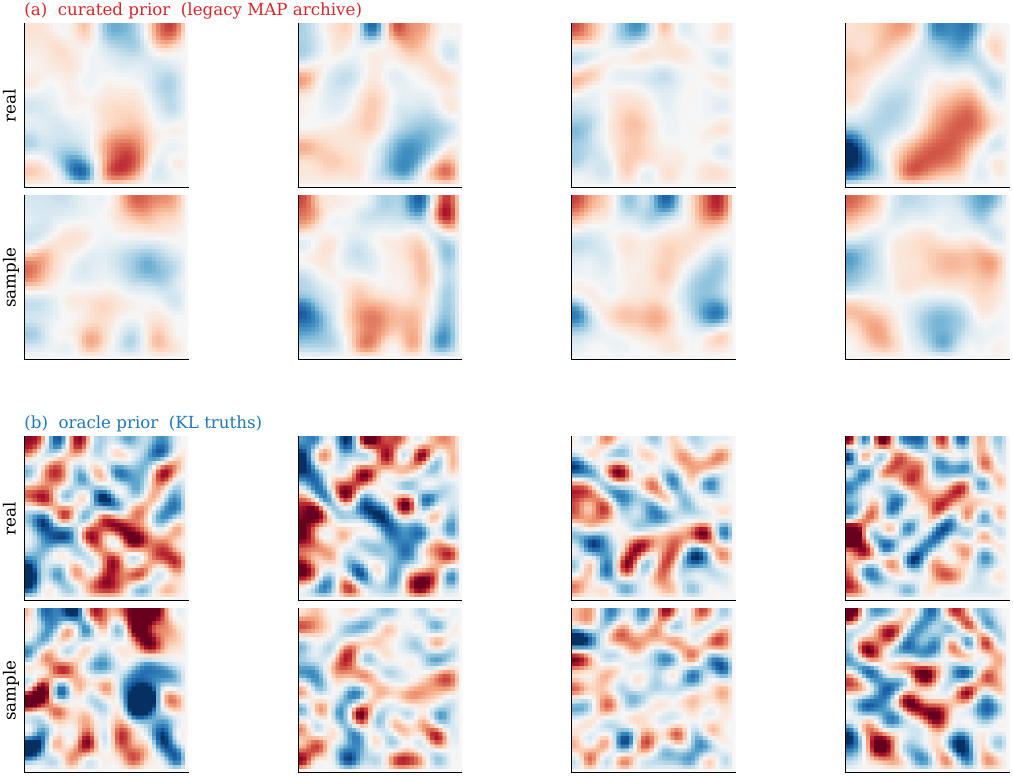}
\caption{For the (a)~\textcolor{madcol}{curated} prior (legacy MAP archive) and the (b)~\textcolor{emcol}{oracle} prior (Karhunen--Lo\`eve truths): real training fields alongside new unconditional flow samples, on one shared color scale.}
\label{fig:darcy-training}
\end{figure}

\begin{figure}[!htbp]
\centering
\includegraphics[width=\linewidth]{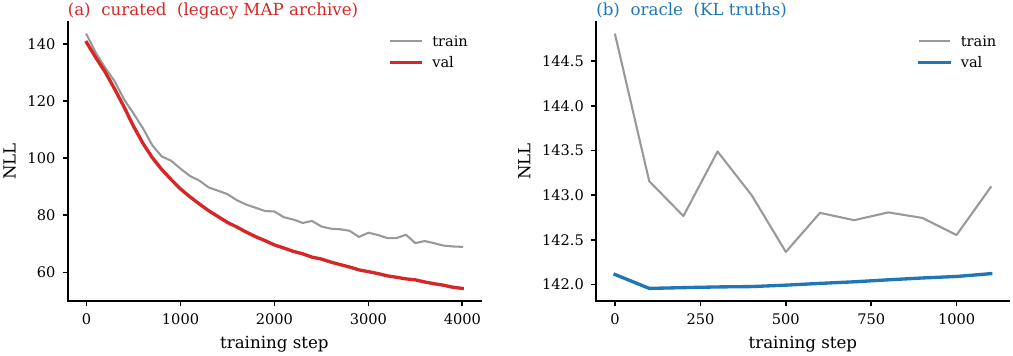}
\caption{Train and validation negative log-likelihood for the (a)~\textcolor{madcol}{curated} and (b)~\textcolor{emcol}{oracle} groundwater flows.}
\label{fig:darcy-loss}
\end{figure}

\paragraph{Posterior sampling.} The two posterior reads use different samplers, and neither reaches the blind subspace with data. On the seismic operator the posterior is drawn by guided diffusion (diffusion posterior sampling)~\citep{chung2023dps}, taking the likelihood score at the current state as guidance; its blind component vanishes,
\begin{equation}
\bN^\top\nabla_{\bx}\log p(\by\mid\bx)=\bN^\top\bA^\top\Gamma^{-1}(\by-\bA\bx)=\bm 0\qquad(\bA\bN=\mathbf 0),
\label{eq:nullscore}
\end{equation}
so the guidance moves the sample only along the resolved directions and the blind report is the prior's own fiber conditional~(\cref{p:blind}); along a near-null direction $\bm v$ the data add at most $(\bA\bm v)^\top\Gamma^{-1}(\bA\bm v)\le\|\bA\bm v\|^2/\lambda_{\min}(\Gamma)$ of precision~(\cref{p:nearnull}), negligible on the blind tail. On the groundwater operator the flow prior is read by preconditioned Crank--Nicolson: reparameterizing the coefficients $\bm\xi=\Psi^{-1}(\bm w)$ through the flow $\Psi$ with $\bm w\sim\mathcal N(\bm 0,\bm I)$, the target in $\bm w$ has a Gaussian reference, so the pCN proposal of \citet{cotter2013mcmc} applies directly, dimension-robust and requiring no gradient of the flow.

\paragraph{A leakage check on the sampler.} The seismic sampler's preconditioner is diagonal in a per-pixel basis, which is exactly the type of reweighting that can propagate a likelihood update into directions the operator does not resolve. We measure, for every certified blind direction, the displacement the guidance produces against the ensemble's own blind band, under three configurations: unpreconditioned, preconditioned, and preconditioned with the scored-blind component projected out at every step. Unpreconditioned guidance moves those directions by a negligible fraction of the band, as the flat likelihood requires; the preconditioner alone moves the worst of them \emph{beyond} the band, for both priors; the projection restores it to negligible. The displayed posteriors use the projected sampler, and the check is what justifies reading their blind report as the prior's.

\begin{figure}[!htbp]
\centering
\includegraphics[width=\linewidth]{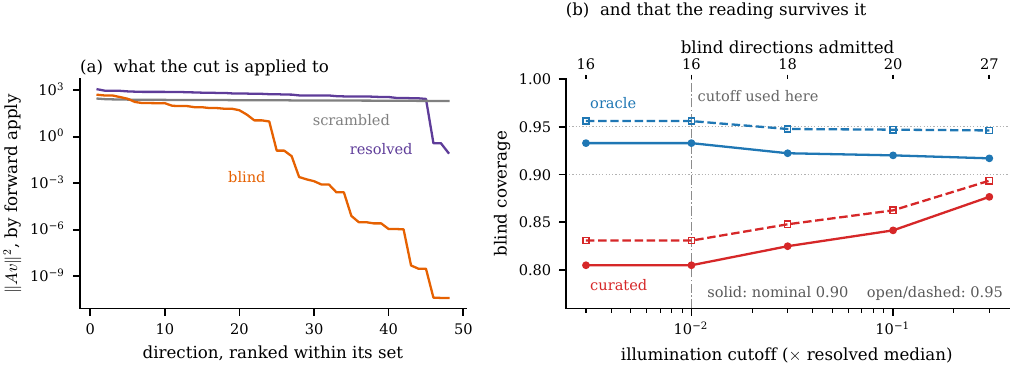}
\caption{(a)~$\|\bA\bm v\|^2$ by forward apply for the $48$ probed directions per set, the quantity the cut is applied to; (b)~blind coverage against the cutoff at two nominal levels, with, along the top, the dimension the admitted directions span.}
\label{fig:seismic-cutoff}
\end{figure}

\paragraph{The cutoff that sets the seismic split.} That split has a free parameter, so the reading must be robust to it~(\cref{fig:seismic-cutoff}). Below the chosen cutoff nothing can move: the certified blind basis is unchanged there, so the reading is invariant by construction rather than by measurement. Above it the subspace grows, and the curated prior's shortfall shrinks as better-illuminated directions (which the data do partly constrain) are admitted alongside the blind ones. The shortfall never closes, and the gap to the truth-trained control remains at every cutoff over two decades. Both priors are scored on the same stored samples at every cutoff, a paired check on the split rather than the per-seed protocol of the main read.

\section{Additional experiments}
\label{app:additional}

\paragraph{The coverage table.} \Cref{tab:coverage} reports the achieved levels shown in \cref{fig:reliability}; the two priors differ on the blind subspace, mildly on the seismic near-null and severely on the groundwater kernel.

\begin{table}[!hb]
\centering
\caption{Achieved central-$(1-\alpha)$ coverage at the shared levels, averaged over the training seeds (per seed in \cref{fig:reliability}).}
\label{tab:coverage}
\small
\begin{tabular}{lllccccc}
\toprule
 & & & \multicolumn{5}{c}{nominal central-$(1-\alpha)$ level} \\
\cmidrule(lr){4-8}
operator & subspace & prior & $.50$ & $.70$ & $.80$ & $.90$ & $.95$ \\
\midrule
Seismic Born & resolved & \textcolor{emcol}{oracle} & $.69$ & $.82$ & $.87$ & $.92$ & $.95$ \\
 &  & \textcolor{madcol}{curated} & $.85$ & $.91$ & $.93$ & $.95$ & $.96$ \\
 & blind & \textcolor{emcol}{oracle} & $.74$ & $.85$ & $.90$ & $.94$ & $.96$ \\
 &  & \textcolor{madcol}{curated} & $.69$ & $.75$ & $.78$ & $.81$ & $.84$ \\
Groundwater & resolved & \textcolor{emcol}{oracle} & $.45$ & $.66$ & $.74$ & $.85$ & $.93$ \\
 &  & \textcolor{madcol}{curated} & $.39$ & $.60$ & $.68$ & $.78$ & $.85$ \\
 & blind & \textcolor{emcol}{oracle} & $.49$ & $.70$ & $.79$ & $.89$ & $.94$ \\
 &  & \textcolor{madcol}{curated} & $.08$ & $.12$ & $.15$ & $.19$ & $.22$ \\
\bottomrule
\end{tabular}
\end{table}

\paragraph{The groundwater posterior on one observation.} What makes this example sharp is that its regularizer is \emph{correctly specified}, since the legacy MAP maximizes the posterior under the true Gaussian prior and the oracle trains on that prior's own samples, so the width mechanism of \cref{p:cover} is absent. The single-best collapse remains~(\cref{p:map}, exact at a kernel and read empirically here per \cref{r:nonlinear}), and the two priors' posteriors still differ~(\cref{fig:darcy-posterior}). It is curation, not misspecification, that empties the blind interval. Since neither a diffusion prior nor a guided sampler is involved here, the shortfall hinges on the estimator's inherited assumption rather than on either.

\paragraph{The added-measurement repair.}
This last reading tests the repair of \cref{c:augment} well beyond the corollary's hypotheses. The operator is the nonlinear groundwater map rather than a linear one, the added measurement is a raw physical reading rather than an instrument built on the kernel, the archive is the single-best one of \cref{p:map} rather than the posterior-sample loop the corollary curates with, and what must move is a trained flow's report rather than a population law's.

The added measurement is six boreholes, each reading the log-permeability at one node, so their rows are exact, linear in the coefficients, and, unlike the head sensors', fixed as the state changes. The nodes are chosen by a greedy design restricted to the operator's exact kernel, maximizing the information the boreholes provide about the blind coordinates. That objective is monotone submodular, so the greedy choice is within a fixed factor of the optimum~\citep{nemhauser1978analysis}. The readings are simulated from the same truths the priors are scored against and contain independent noise at $\sigma_u=0.04$ in Karhunen--Lo\`eve units. The exact null is $69$-dimensional, taken from the exact adjoint Jacobian. The boreholes split it into the $6$ directions they resolve and the $63$ they do not, which serve as the experiment's control.

Curating jointly rebuilds the archive under sensors and boreholes together, and each record keeps its stored sensor observation, so the borehole reading is the only thing that changes. The joint reconstruction requires a larger iteration budget, at which the sensors-only archive scores the same. The three flows, trained on the sensors-only archive, the joint archive, and the truths, are matched in architecture, budget, and dequantization. The dequantization is set in absolute units, since the jointly curated archive is wider and would otherwise be smoothed more on exactly the directions that must not move. Each is trained with three seeds, and \cref{fig:darcy-augment} shows one. One curated seed fits its archive markedly worse than the other two and so reports a wider blind spread, which errs against the comparison rather than in its favor.

\begin{figure}[!htb]
\centering
\includegraphics[width=\linewidth]{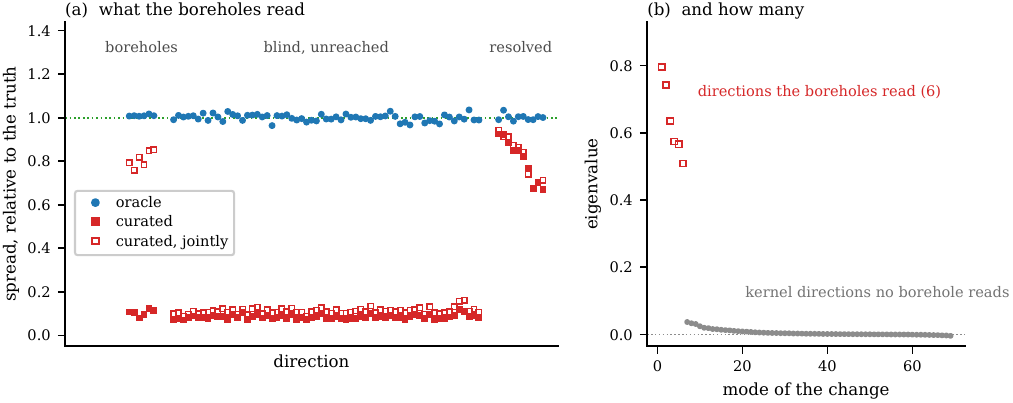}
\caption{Joint curation against an added measurement. (a)~each direction's reported spread against the \textcolor{healcol}{truth}'s (dotted), sensors-only (filled) and jointly with the boreholes (open). (b)~the modes of the change in the report, restricted to the exact kernel.}
\label{fig:darcy-augment}
\end{figure}

In \cref{fig:darcy-augment}a the directions the boreholes read, collapsed under the sensors alone, come back toward the truth's spread. The directions they never reach remain collapsed while the resolved subspace moves only by the small amount the joint estimate predicts, and the oracle matches the truth throughout. Read without choosing a basis, the change in the report restricted to the exact kernel has as many non-negligible modes as there are boreholes and then decays~(\cref{fig:darcy-augment}b). The archive lies where the joint estimate's linear--Gaussian algebra predicts, and its corrected directions stop short of the truth by the shrinkage a single-best estimate keeps~(\cref{p:map}) rather than by a limit of the repair.

The directions the boreholes never reach do increase slightly, far less than the corrected ones move, and two controls attribute that increase to the estimator rather than to information arriving. A placebo measurement that provides no blind information (the same rows projected onto the resolved subspace) moves them further than the real boreholes do. Varying the boreholes' noise from sharp to nearly uninformative leaves them flat while the corrected directions degrade smoothly, whereas leakage would grow with the measurement's strength. Both controls rebuild the archive exactly as the main experiment does, and what remains is the smoothing scale a flow fit to a collapsed archive reports~(\cref{r:map-lg}). The split is the base-state kernel, read as fixed~(\cref{r:nonlinear}), since across the archive's own fields the sensors see these directions no better than their own noise level, and the sub-noise tail is set aside here because its blind status changes with the state.

\section{Reproduction}
\label{app:repro}

The code accompanying this paper---the scripts that reproduce every figure and record, the resolvability statement and certified band of Section~\ref{sec:reportcard}, and the Lean~4 development---is available at \url{https://github.com/luqigroup/resolvability}; the archived datasets, operators, and checkpoints it requires are fetched automatically on first use.

\end{document}